%% file: main.tex
\documentclass[sigconf]{acmart}
\pdfoutput=1
\renewcommand\footnotetextcopyrightpermission[1]{} 
\usepackage{booktabs} 
\usepackage{mathrsfs}
\usepackage{amsfonts,amssymb,bm}
\usepackage{amsmath}
\usepackage{graphicx}
\usepackage{color}
\usepackage{natbib}
\setcitestyle{square,numbers,sort&compress}

\usepackage{multirow}
\usepackage{graphicx}
\usepackage{makecell}
\usepackage{verbatim}
\usepackage{subfigure}
\usepackage[para,online,flushleft]{threeparttable}

\setcopyright{rightsretained}

\acmDOI{10.475/123_4}

\acmISBN{123-4567-24-567/08/06}

\acmConference[WOODSTOCK'97]{ACM Woodstock conference}{July 1997}{El
  Paso, Texas USA}
\acmYear{1997}
\copyrightyear{2016}

\acmArticle{4}
\acmPrice{15.00}

\begin{document}

\title{An ADMM-Based Universal Framework for Adversarial Attacks on Deep Neural Networks}


\author{{\huge Pu Zhao$^{1}$, Sijia Liu$^2$, Yanzhi Wang$^1$, Xue Lin$^1$}
{\\
$^1$Department of ECE, Northeastern University\\
$^2$MIT-IBM Watson AI Lab, IBM Research AI\\
}} 

\begin{abstract}
Deep neural networks (DNNs) are known vulnerable to adversarial attacks. That is, adversarial examples, obtained by adding delicately crafted distortions onto original legal inputs, can mislead a DNN to classify them as any target labels. In a successful adversarial attack, the targeted mis-classification should be achieved with the minimal distortion added. In the literature, the added distortions are usually measured by $L_0$, $L_1$, $L_2$, and $L_{\infty}$ norms, namely, $L_0$, $L_1$, $L_2$, and $L_{\infty}$ attacks, respectively. However, there lacks a versatile framework for all types of adversarial attacks.

This work for the first time unifies the methods of generating adversarial examples by leveraging ADMM (Alternating Direction Method of Multipliers), 
an operator splitting optimization approach, such that $L_0$, $L_1$, $L_2$, and $L_{\infty}$ attacks can be 
effectively implemented by this general framework with little modifications. Comparing with the state-of-the-art attacks in each category, our ADMM-based attacks are so far the strongest, achieving both the 100\% attack success rate and the minimal distortion. 


\end{abstract}

%
%



\maketitle

\input{samplebody-conf}

\bibliographystyle{ACM-Reference-Format}
\bibliography{egbib}

\end{document}

%% file: samplebody-conf.tex
\section{Introduction}

Deep learning has been demonstrating exceptional performance on several categories of machine learning problems and has been applied in many settings \cite{deng2009imagenet, krizhevsky2012imagenet, taigman2014deepface,he2016deep,hinton2012deep,silver2016mastering,makantasis2015deep}. However, people recently find that deep neural networks (DNNs) could be vulnerable to adversarial attacks \cite{carlini2016hidden,nguyen2015deep,kurakin2016adversarial}, which arouses concerns of applying deep learning in security-critical tasks. Adversarial attacks are implemented through generating adversarial examples, which are crafted by adding delicate distortions onto legal inputs. Fig. \ref{fig:adversarial_examples} shows adversarial examples for targeted adversarial attacks that can fool DNNs. 


The security properties of deep learning have been investigated from two aspects: (i) enhancing the robustness of DNNs under adversarial attacks and (ii) crafting adversarial examples to test the vulnerability of DNNs. For the former aspect, research works have been conducted by either filtering out added distortions \cite{guo2017countering,bhagoji2017dimensionality,dziugaite2016study,xie2017mitigating} or revising DNN models \cite{papernot2016distillation,dhillon2018stochastic,feinman2017detecting} to defend against adversarial attacks. For the later aspect, adversarial examples have been generated heuristically \cite{goodfellow2014explaining,su2017one}, iteratively \cite{papernot2016limitations, kurakin2016adversarial,hong2017linear,wang2014bregman}, or by solving optimization problems \cite{szegedy2013intriguing,carlini2017towards,chen2017ead,athalye2018obfuscated}. These two aspects mutually benefit each other towards hardening DNNs under adversarial attacks. And our work deals with the problem from the later aspect.
 

For targeted adversarial attacks, the crafted adversarial examples should be able to mislead the DNN to classify them as any target labels, as done in Fig. \ref{fig:adversarial_examples}. Also, in a successful adversarial attack, the targeted mis-classification should be achieved with the minimal distortion added to the original legal input. Here comes the question of how to measure the added distortions. Currently, in the literature, $L_0$, $L_1$, $L_2$, and $L_\infty$ norms are used to measure the added distortions, and they are respectively named $L_0$, $L_1$, $L_2$, and $L_\infty$ adversarial attacks. Even though no measure can be perfect for human perceptual similarity, these measures or attack types may be employed for different application specifications. This work bridges the literature gap by unifying all the types of attacks with a single intact framework. 



\begin{figure}[htbp]\centering                                                      
\subfigure[MNIST]{                    
\begin{minipage}[t]{0.22\textwidth}
\centering                                                         
\includegraphics[width=1\textwidth]{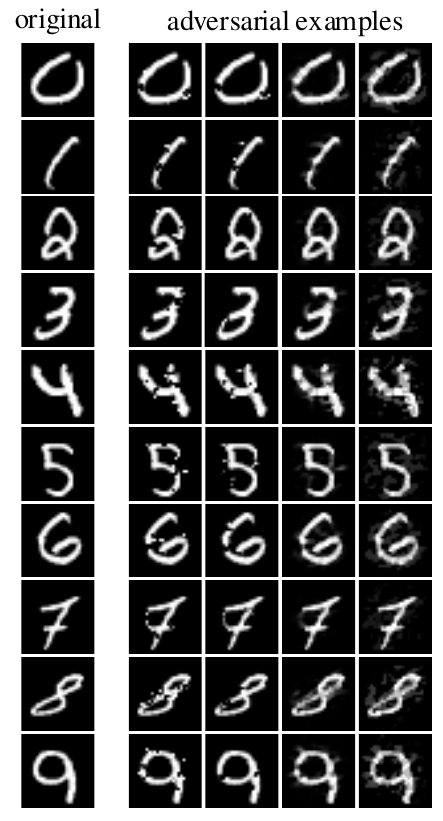}              
\end{minipage}
}
\subfigure[CIFAR-10]{                 
\begin{minipage}[t]{0.22\textwidth}
\centering                                                          
\includegraphics[width=0.93\textwidth]{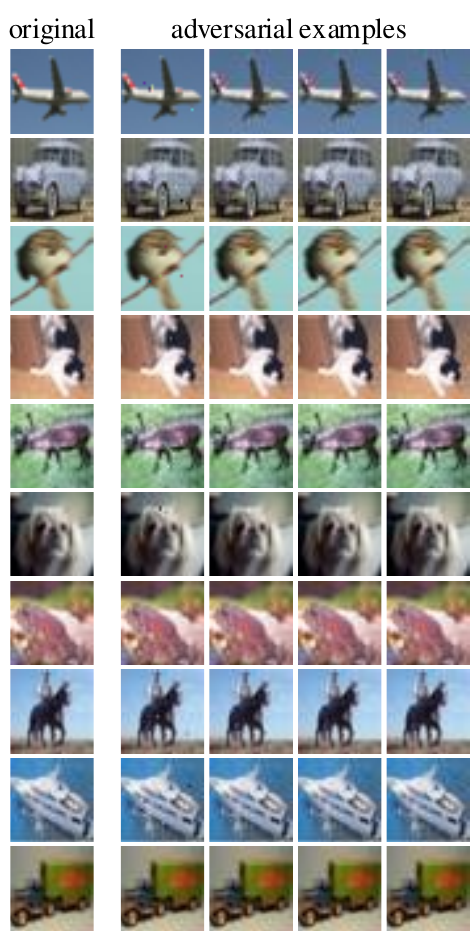}              
\end{minipage}
}
\caption{Adversarial examples generated by our ADMM $L_0$, $L_1$, $L_2$, and $L_\infty$ attacks for MNIST (left) and CIFAR-10 (right) datasets. The leftmost column contains the original legal inputs. The next four columns are the corresponding adversarial examples crafted using our ADMM $L_0$, $L_1$, $L_2$, and $L_\infty$ attacks, respectively. If the original inputs are correctly classified as label $l$, then the adversarial examples mislead the DNN to classify them as target label $l+2$.}                      
\label{fig:adversarial_examples}                                                        
\end{figure}

In order to benchmark DNN defense techniques and to push for a limit of the DNN security level, we should develop the strongest adversarial attacks. For this purpose, we adopt the white-box attack assumption in that the attackers have complete information about the DNN architectures and all the parameters. This is also a realistic assumption, because even if we only have black-box access to the DNN model, we can train a substitute model and transfer the attacks generated using the substitute model. And for the same purpose, we adopt the optimization-based approach to generate adversarial examples. The objectives of the optimization problem should be (i) misleading the DNN to classify the adversarial example as a target label and (ii) minimizing the $L_p$ norm of the added distortion.

By leveraging ADMM (Alternating Direction Method of Multipliers) \cite{boyd2011distributed}, 
an operator splitting optimization approach,
we provide a universal framework for $L_0$, $L_1$, $L_2$, and $L_\infty$ adversarial attacks. ADMM decomposes an original optimization problem into two correlated subproblems, 
each of which can be solved more efficiently or analytically,
and then coordinates solutions to the subproblems to construct a solution to the original problem. This decomposition-alternating
procedure of ADMM blends the benefits of dual decomposition and augmented Lagrangian for solving problems with non-convex and combinatorial constraints. Therefore, ADMM introduces no additional sub-optimality besides the original gradient-based backpropagation method commonly used in DNNs and provides a faster linear
convergence rate than state-of-the-art iterative attacks \cite{papernot2016limitations, kurakin2016adversarial,hong2017linear,wang2014bregman}. We also compare with the optimization-based approaches, i.e., Carlini \& Wagner (C\&W) attack \cite{carlini2017towards} and Elastic-net (EAD) attack \cite{chen2017ead}, which are the currently strongest attacks in the literature.   


The major contributions of this work and its differences from C\&W and EAD attacks are summarized as follows: 
\begin{itemize}
\item With our ADMM-based universal framework, all the $L_0$, $L_1$, $L_2$, and $L_\infty$ adversarial attacks can be implemented with little modifications, while C\&W only performs $L_0$, $L_2$, and $L_\infty$ attacks and EAD only performs $L_1$ and $L_2$ attacks.
\item C\&W $L_0$ attack needs to run their $L_2$ attack iteratively to find the pixels with the least effect and fix them, thereby identifying a minimal subset of pixels for modification to generate an adversarial example.
\item C\&W $L_\infty$ attack through naively optimization with gradient descent may produce very poor initial results. They solve the issue by introducing a limit on the $L_\infty$ norm and reducing the limit iteratively. 
\item EAD attack  minimizes a weighted sum of $L_1$ and $L_2$ norms. However, a universal attack generation model is missing.
\item Our extensive experiments show that we are so far the best attacks. Besides the 100\% attack success rate, our ADMM-based attacks outperform C\&W and EAD in each type of attacks in terms of minimal distortion.
\end{itemize}

Besides comparing with C\&W, EAD and other attacks, we also test our attacks against defenses such as defensive distillation \cite{papernot2016distillation} and adversarial training \cite{tram2018ensemble}, demonstrating the success of our attacks. In addition, we validate the transferability of our attacks onto different DNN models. 
The codes of our attacks to reproduce the results are available online\footnotemark.
\footnotetext{Codes will be available upon publication of this work.}

\section{Related Work}

We introduce the most representative attacks and defenses in this section.

\subsection{Adversarial Attacks}
\textbf{L-BFGS Attack \cite{szegedy2013intriguing}} is the first optimization-based attack and is an $L_2$ attack that uses $L_2$ norm to measure the distortion in the optimization objective function.

\noindent\textbf{JSMA Attack \cite{papernot2016limitations}} is an $L_0$ attack and uses a greedy algorithm that picks the most influential pixels by calculating Jacobian-based Saliency Map and modifies the pixels iteratively. The computational complexity is prohibitive even for applying to ImageNet dataset.


\noindent\textbf{FGSM \cite{goodfellow2014explaining} and IFGSM \cite{kurakin2016adversarial} Attacks} are $L_\infty$ attacks and utilize the gradient of the loss function to determine the direction to modify the pixels. They are designed to be fast, rather than optimal. They can be used for adversarial training by directly changing the loss function instead of explicitly injecting adversarial examples into the training data. The fast gradient method (FGM) and the iterative fast gradient method (IFGM) are improvements of FGSM and IFGSM, respectively, that can be fitted as $L_1$, $L_2$, and $L_\infty$ attacks.


\noindent\textbf{C\&W Attacks \cite{carlini2017towards}} are a series of $L_0$, $L_2$, and $L_\infty$ attacks that achieve 100\% attack success rate with much lower distortions comparing with the above-mentioned attacks. In particular, the C\&W $L_2$ attack is superior to L-BFGS attack (which is also an $L_2$ attack) because it uses a better objective function.

\noindent\textbf{EAD Attack \cite{chen2017ead}} formulates the process of crafting adversarial examples as an elastic-net regularized optimization problem. Elastic-net regularization is a linear mixture of $L_1$ and $L_2$ norms used in the penalty function. EAD attack is able to craft $L_1$-oriented adversarial examples and includes the C\&W $L_2$ attack as a special case.
 
\subsection{Representative Defenses }
\textbf{Defensive Distillation \cite{papernot2016distillation}} introduces \emph{temperature} into the softmax layer and uses a higher temperature for training and a lower temperature for testing. The training phase first trains a teacher model that can produce soft labels for the training dataset and then trains a distilled model using the training dataset with soft labels. The distilled model with reduced temperature will be preserved for testing.

\noindent\textbf{Adversarial Training \cite{tram2018ensemble}} injects adversarial examples with correct labels into the training dataset and then retrains the neural network, thus increasing robustness of DNNs under adversarial attacks.

\section{An ADMM-Based Universal Framwork for Adversarial Attacks} \label{ADMMframework}


ADMM was first introduced in the mid-1970s with roots in the 1950s, and the algorithm and theory have been established by the mid-1990s. It was proposed and made popular recently by S. Boyd et al. for statistics and machine learning problems with a very large number of features or training examples \cite{boyd2011distributed}. 
ADMM method takes the form of a decomposition-alternating 
procedure, in which the solutions to small local subproblems are coordinated to find a solution to a large global problem. It can be viewed as an attempt to blend the benefits of dual decomposition and augmented Lagrangian methods for constrained optimization.

ADMM was developed in part to bring robustness to the dual ascent method, and in particular, to yield convergence without assumptions like strict convexity or finiteness of the objective.
ADMM is also capable of dealing with combinatorial constraints due to its decomposition property. It can be used in many practical applications, where the convexity of the objective can not be guaranteed or it has some combinatorial constraints. Besides, 
it converges fast in many cases since the two arguments are updated in an alternating or sequential fashion, which accounts for the term \textit{alternating direction}.

\subsection{Notations and Definitions}

In this paper, we mainly evaluate the adversarial attacks with image classification tasks. A two dimensional vector ${\bm x} \in {\mathop{\rm R}\nolimits} ^{hw}$ represents a gray-scale image with height $h$ and width $w$. For a colored RGB image with three channels, a three dimensional tensor ${\bm x} \in {\mathop{\rm R}\nolimits} ^{3hw}$ is utilized to denote it.
Each element $x_i$ represents the value of the $i$-th pixel and is scaled to the range of $[0,1]$.  
A neural network has the model $F(\bm{x})=\bm{y}$, where $F$ generates an output $\bm{y}$ given an input $\bm{x}$. Model $F$ is fixed since we perform attacks on given neural network models.

The output layer performs softmax operation and the neural network is an $m$-class classifier. Let the logits $Z(\bm{x})$ denote the input to the softmax layer, which represents the output of all layers except for the softmax layer. We have $F({\bm{x}}) = {\rm{softmax}}(Z({\bm{x}})) = {\bm{y}}$. The element ${y_i}$ of the output vector ${\bm{y}}$ represents the probability that input $\bm{x}$ belongs to the $i$-th class. The output vector ${\bm{y}}$ is treated as a probability distribution, and its elements satisfy $0 \le {y_i} \le 1$ and $y_1+y_2+\dots+y_m=1$. The neural network classifies input $\bm{x}$ according to the maximum probability, i.e., $C(\bm{x}) = \arg \mathop {\max }\limits_i {y_i}$.


The adversarial attack can be either targeted or untargeted. Given an original legal input $\bm{x_0}$ with its correct label $t^*$, the untargeted adversarial attack is to find an input $\bm{x}$ satisfying $C(\bm{x}) \ne t^*$ while $\bm{x}$ and $\bm{x_0}$ are close according to some measure of the distortion. The untargeted adversarial attack does not specify any target label  to mislead the classifier.
In the targeted adversarial attack, with a given target label $t \ne {t^*}$, an adversarial example is an input $\bm{x}$ such that $C(\bm{x}) = t$ while $\bm{x}$ and $\bm{x_0}$ are close according to some measure of the distortion. 
In this work, we consider targeted adversarial attacks since they are believed stronger than untargeted attacks.

\subsection{General ADMM Framework for Adversarial Attacks}

The \textit{initial} problem of constructing adversarial examples is defined as: \textbf{Given} an original legal input image $\bm{x_0}$ and a target label $t$, \textbf{find} an adversarial example $\bm{x}$, \textbf{such that} $\mathcal{D}(\bm{x}-\bm{x_0})$ is minimized, $C(\bm{x}) = t$, and $\bm{x} \in{[0,1]^n}$. $\bm{x}-\bm{x_0}$ is the distortion added onto the input $\bm{x_0}$.  $C(\cdot)$ is the classification function of the neural network and the adversarial example $\bm{x} \in {[0,1]^n}$ is classified as the target label $t$. 


$\mathcal{D}(\bm{x}-\bm{x_0})$ is a measure of the distortion $\bm{x}-\bm{x_0}$. We need to measure the distortion between the original legal input $\bm{x_0}$ and the adversarial example $\bm{x}$. ${L_p}$ norms are the most commonly used measures in the literature. The ${L_p}$ norm of the distortion between $\bm{x}$ and $\bm{x_0}$ is defined as:
\begin{equation}
{\left\| {\bm{x}-\bm{x_0}} \right\|_p} = {\left( {\sum\limits_{i = 1}^n {{{\left| \bm{x}_{i}-\bm{x}_{0i} \right|}^p}} } \right)^{\frac{1}{p}}}
\end{equation}
We see the use of $L_0$, $L_1$, $L_2$, and $L_\infty$ norms in different attacks. 
\begin{itemize}
\item[-] \emph{${L_0}$ norm}: measures the number of mismatched elements between $\bm{x}$ and $\bm{x_0}$. 
\item[-] \emph{${L_1}$ norm}: measures the sum of the absolute values of the differences between $\bm{x}$ and $\bm{x_0}$.  
\item[-] \emph{${L_2}$ norm}: measures the standard Euclidean distance between $\bm{x}$ and $\bm{x_0}$. 
\item[-] \emph{${L_\infty }$ norm}: measures the maximum difference between $\bm{x}_i$ and $\bm{x}_{0i}$ for all $i$'s.
\end{itemize}

In this work, with a general ADMM-based framework, we implement $L_0$, $L_1$,  $L_2$, and $L_\infty$ attacks, respectively. When generating adversarial examples in the four attacks, $\mathcal{D}(\bm{x}-\bm{x_0})$ in the objective function becomes $L_0$, $L_1$,  $L_2$, and $L_\infty$ norms, respectively. For the simplicity of expression, in the general ADMM-based framework, the form of $\mathcal{D}(\bm{x}-\bm{x_0})$ is used to denote the measure of $\bm{x}-\bm{x_0}$. When introducing the detailed four attacks based on the ADMM framework, we utilize the form of $L_p$ norm to represent the distortion measure.  

ADMM provides a systematic way to deal with non-convex and combinatorial constraints by breaking the initial problem into two subproblems. To do this, the initial problem is first transformed into the following problem, introducing an auxiliary variable $\bm{z}$: 
\begin{equation}\label{original_problem}
\begin{array}{l}
\mathop {\min }\limits_{\bm{x} ,\bm{z}} \;\, {\kern 1pt} \mathcal{D}(\bm{x} -\bm{x_0} ) + g(\bm{z})\\
s.t.\quad \bm{x}  = \bm{z}\\
\quad \quad \ \bm{z} \in {[0,1]^n}
\end{array}
\end{equation}
where $g(\bm{x})$ has the form:
\begin{equation}
g(\bm{x}) = \left\{ {\begin{array}{*{20}{c}}
	0&{\text{if}\;\mathop {\max }\limits_{i \ne t} (Z(\bm{x})_i) - Z{{(\bm{x})}_t} \le 0}\\
	{ + \infty }&{\text{otherwise}}
	\end{array}} \right.
\end{equation}
Here $Z(\bm{x})$ is the logits before the softmax layer. $Z(\bm{x})_i$ means the $i$-th element of $Z(\bm{x})$. The function $g(\bm{x})$ ensures that the input is classified with target label $t$.  
The augmented Lagrangian function of problem (\ref{original_problem}) is as follows:
\begin{equation}
{L_\rho }(\bm{x} ,\bm{z},\bm{u}) = \mathcal{D}(\bm{x} - \bm{x_0} ) + g(\bm{z}) + {\bm{u}^T}(\bm{x}  - \bm{z}) + \frac{\rho }{2}\left\| {\bm{x}  - \bm{z}} \right\|_2^2 
\end{equation}
where $\bm{u}$ is the dual variable or Lagrange multiplier and $\rho > 0 $ is called the penalty parameter. 
Using the scaled form of ADMM by defining $\bm{u} = \rho \bm{s}$, we have:
\begin{equation}\label{Lro}
{L_\rho }(\bm{x} ,\bm{z},\bm{s}) = \mathcal{D}(\bm{x}- \bm{x_0} ) + g(\bm{z}) + \frac{\rho }{2}\left\| {\bm{x}  - \bm{z} + \bm{s}} \right\|_2^2 - \frac{\rho }{2}\left\| \bm{s} \right\|_2^2
\end{equation}

ADMM solves problem (\ref{original_problem}) through iterations. In the $k$-th iteration, the following steps are performed:
\begin{equation}\label{It1}
{{\bm{x} ^{k + 1}} = {\arg \mathop {\min }\limits_{\bm{x}}}  \quad {L_\rho }(\bm{x} ,{\bm{\bm{z}}^k},{\bm{s}^k})}
\end{equation}
\begin{equation}\label{It2}
{\bm{z}^{k + 1}} = \arg \mathop {\min }\limits_{\bm z} \quad {L_\rho }({\bm{x} ^{k + 1}},\bm{z},{\bm{s}^k})
\end{equation}
\begin{equation}\label{update1}
{\bm{s}^{k + 1}} = {\bm{s}^k} + {\bm{x} ^{k + 1}} - {\bm{z}^{k + 1}}
\end{equation}
In Eqn. (\ref{It1}), we find $\bm{x} ^{k + 1}$ which minimizes $L_\rho$ with fixed $\bm{z}^k$ and $\bm{s}^k$. Similarly, in Eqn. (\ref{It2}), $\bm{x}^{k+1}$ and $\bm{s}^k$ are fixed and we find $\bm{z}^{k+1}$ minimizing $L_\rho$. $\bm{s}^{k+1}$ is then updated accordingly. Note that the two variables $\bm{x}$ and $\bm{z}$ are updated in an alternating or sequential fashion, from which the term \textit{alternating direction} comes. It converges when:
\begin{equation}
\left\| {{\bm{x} ^{k + 1}} - {\bm{z}^{k + 1}}} \right\|_2^2 \le \varepsilon ,\quad \left\| {{\bm{z}^{k + 1}} - {\bm{z}^k}} \right\|_2^2 \le \varepsilon 
\end{equation}

Equivalently, in each iteration, we solve two optimization subproblems corresponding to Eqns. (\ref{It1}) and (\ref{It2}), respectively:
\begin{equation}\label{Sub1}
{\mathop {\min }\limits_{\bm x}}  \quad  \mathcal{D}(\bm{x}-\bm{x_0} ) + \frac{\rho }{2}\left\| {\bm{x}  - \bm{z} + \bm{s}} \right\|_2^2
\end{equation}
and
\begin{equation}\label{Sub2}
\mathop {\min }\limits_{\bm z} \quad g( \bm{z}) + \frac{\rho }{2}\left\| {\bm{x}  - \bm{z} + \bm{s}} \right\|_2^2
\end{equation}
The non-differentiable $g(\bm{x})$ makes it difficult to solve the second subproblem (\ref{Sub2}). Therefore, a new differentiable $g(\bm{x})$ inspired by \cite{carlini2017towards} is utilized as follows:   
\begin{equation}\label{newg}
g(\bm{x}) = c \cdot \max \left( {\left( {\mathop {\max }\limits_{i \ne t} \left( {Z(\bm{x})_i} \right) - Z{{(\bm{x})}_t}} \right),-\kappa } \right) 
\end{equation}
Then, stochastic gradient decent methods can be used to solve this subproblem. The Adam optimizer \cite{KingmaB2015adam} is applied due to its fast and robust convergence behavior. In the new $g(\bm{x})$ of Eqn. (\ref{newg}), $\kappa$ is a confidence parameter denoting the strength of adversarial example transferability. The larger $\kappa$, the stronger transferability of the adversarial example. It can be kept as 0 if we do not evaluate the transferability. 


\subsection{Box Constraint}
The constraint on $\bm{z}$ i.e., $ \bm z \in [0,1]^n$ is known as a ``box constraint'' in the optimization literature. We use a new variable $\bm{w}$ and instead of optimizing over $\bm z$ defined above, we optimize over $\bm w$, based on:
\begin{equation}\label{Box}
\bm{z}= \frac12 \big(\tanh(\bm{w}) + 1\big) 
\end{equation}
Here the $\tanh(\cdot)$ is performed elementwise. Since $-1 \le \tanh(w_i) \le 1$, the method will automatically satisfy the box constraint and allows us to use optimization algorithms that do not natively support box constraints.


\subsection{Selection of Target Label}

For targeted attacks, there are different ways to choose the target labels:
\begin{itemize}
\item[-] \emph{Average Case}: select at random the target label uniformly among all the labels that are not the correct label.
\item[-] \emph{Best Case}: perform attacks using all incorrect labels, and report the target label that is the least difficult to attack.
\item[-] \emph{Worst Case}: perform attacks using all incorrect labels, and report the target label that is the most difficult to attack.
\end{itemize}
We evaluate the performs of the proposed ADMM attacks in the three cases mentioned above.

\subsection{Discussion on Constants}


There are two constants $c$ and $\rho$ in the two subproblems (\ref{Sub1}) and (\ref{Sub2}). Different policies are adopted for choosing appropriate $c$ and $\rho$ in $L_0$, $L_1$, $L_2$ and $L_\infty$ attacks. 
In $L_2$ attack, since $\rho$ acts in both problems (\ref{Sub1}) and (\ref{Sub2}), we fix $\rho$ and change $c$ to improve the solutions. We find that the best choice of $c>0$ is the smallest one that can help achieve $g(\bm{x}) = 0$ in the subproblem (\ref{Sub2}).
Thus, a modified binary search is used to find a satisfying $c$. 
For the ADMM $L_0$ attack, as $\rho$ has stronger and more direct influence on the solutions, $c$ is fixed and adaptive search of $\rho$ is utilized. More details are provided in Section \ref{ro_in_L0}. 
For the ADMM $L_1$ and $L_\infty$ attacks, as we find fixed $c$ and $\rho$ can achieve good performance,  $c$ and $\rho$ are kept unchanged and adaptive search method is not used.

\section{Instantiations of $L_0$, $L_1$, $L_2$ and $L_\infty$ Attacks based on ADMM Framework}

The ADMM framework for adversarial attacks now need to solve two subproblems (\ref{Sub1}) and (\ref{Sub2}). The difference between $L_0$, $L_1$, $L_2$ and $L_\infty$ attacks lies in the subproblem (\ref{Sub1}), while the processes to find the solutions of the subproblem (\ref{Sub2}) based on stochastic gradient descent method are the very similar for the four attacks.

\subsection{$L_2$ Attack}
For $L_2$ attack, the subproblem (\ref{Sub1}) has the form:
\begin{equation}\label{l2a}
{\mathop {\min }\limits_{\bm{x}}}  \quad \left\| \bm{x} -\bm{x_0}  \right\|_2^2 + \frac{\rho }{2}\left\| {\bm{x}  - \bm{z} + \bm{s}} \right\|_2^2
\end{equation}
the solution to which can be directly derived in an analytical format:
\begin{equation}\label{l2s}
{\bm{x}} = \frac{\rho }{{2 + \rho }}({\bm{z}} - {\bm{s}})+\frac{2 }{{2 + \rho }} \bm{x_0}
\end{equation}


Then the complete solution to the $L_2$ attack problem using the ADMM framework is as follows: for the $k$-th iteration,
\begin{equation}\label{sub1ininteration}
{\bm{x} ^{k + 1}} = \frac{\rho }{{2 + \rho }}\bigg({\Big(\frac12 \big(\tanh(\bm{w}^k) + 1\big) \Big)} - {\bm{s}^k}\bigg)+\frac{2 }{{2 + \rho }} \bm{x_0}
\end{equation}
\begin{equation}\label{adm1}
\begin{array}{l}
{{\bm{w}}^{k + 1}} = \arg \mathop {\min }\limits_{\bm{w}} \quad \left( {g\left( {\frac{1}{2}(\tanh ({\bm{w}}) + 1)} \right)} \right.\\
\quad \quad \;\;\;\ \ \left. { + \frac{\rho }{2}\left\| {{{\bm{x }}^{k + 1}} - \left( {\frac{1}{2}(\tanh ({\bm{w}}) + 1) } \right) + {{\bm{s}}^k}} \right\|_2^2} \right)
\end{array}
\end{equation}
\begin{equation}\label{supdate}
{\bm{s}^{k + 1}} = {\bm{s}^k} + {\bm{x} ^{k + 1}} - {\Big(\frac12 \big(\tanh(\bm{w}^{k + 1}) + 1\big)  \Big)}
\end{equation}
Eqn. (\ref{sub1ininteration}) corresponds to the analytical solution to the subproblem (\ref{Sub1}) i.e., problem (\ref{l2a}) with Eqn. (\ref{Box}) replacing $\bm z$ in Eqn. (\ref{l2s}).
Eqn. (\ref{adm1}) corresponds to the subproblem (\ref{Sub2}) with Eqn. (\ref{Box}) replacing $\bm z$ and $g$ taking the form of Eqn. (\ref{newg}).  
The solution to Eqn. (\ref{adm1}) is derived through the Adam optimizer with stochastic gradient descent.

\subsection{$L_0$ Attack} \label{ro_in_L0}

For $L_0$ attack, the subproblem (\ref{Sub1}) has the form:
\begin{equation}\label{l0a}
\mathop {\min }\limits_{\bm{x}}  \quad \left\| \bm{x} -\bm{x_0} \right\|_{0} + \frac{\rho }{2}\left\| {\bm{x}  - \bm{z} + \bm{s}} \right\|_2^2
\end{equation}
Its equivalent optimization problem is as follows:
\begin{equation}\label{equl0a}
\mathop {\min }\limits_{\bm{\delta}}  \quad \left\| \bm{\delta} \right\|_{0} + \frac{\rho }{2}\left\| {\bm{\delta}  - \bm{z} + \bm{s} +\bm{x_0}} \right\|_2^2
\end{equation}
The solution to problem (\ref{l0a}) can be obtained through $\bm{x^*} = \bm{x_0} + \bm{\delta^*}$ where $\bm{\delta^*} $ is the solution to problem (\ref{equl0a}).
The solution to problem (\ref{equl0a}) can be derived in this way:
let $\bm{\delta}$ be equal to $\bm{z}-\bm{s}-\bm{x_0}$ first, then for each element in $\bm{\delta}$, if its square is smaller than $\frac{2}{\rho}$, make it zero. 
A proof for the solution is given in the following.


\begin{lemma}\label{lemma}
	Suppose that two matrices $\bm{A}$, $\bm{B}$ are of the same size, and that there are at least $k$ zero elements in $\bm{A}$. Then the optimal value of the following problem is the sum of the square of the $k$ smallest elements in $\bm{B}$.
	\begin{equation}\label{le}
	\mathop {\min }\limits_{\bm{A}}  \quad \left\| {\bm{A} -\bm{B}} \right\|_2^2
	\end{equation}
\end{lemma}
The proof for the lemma is straightforward and we omit it for the sake of brevity. We use $h(\bm{x},k)$ to denote the sum of 
the $k$ smallest $\bm x_i^2$ ($\bm x_i$ is an element in $\bm x$). 
\begin{theorem}\label{theorem}
	Set $\bm{\delta}=\bm{z}-\bm{s}-\bm{x_0}$ and then make those elements in $\bm{\delta}$ zeros if their square are smaller than $\frac{2}{\rho}$. Such $ \bm{\delta}$ would yield the minimum objective value of problem (\ref{equl0a}).
\end{theorem}
\begin{proof}
	Suppose that $\bm{\delta_1}$ is constructed according to the above rule in Theorem 1, and $\bm{\delta_1}$ has $k_1$ elements equal to 0. We need to prove that $\bm{\delta_1}$ is the optimal solution with the minimum objective value. Suppose we have another arbitrary solution  $\bm{\delta_2}$ with $k_2$ elements equal to 0. Both $\bm{\delta_1}$ and  $\bm{\delta_2}$ have $n$ elements. The objective value of solution $\bm{\delta_1}$ is:
	\begin{equation}
	\left\| {\bm{\delta}_1} \right\|_0 + \frac{\rho }{2}h(\bm{z}-\bm{s}-\bm{x_0}, k_1) = n-k_1 + \frac{\rho }{2}h(\bm{z}-\bm{s}-\bm{x_0}, k_1) 
	\end{equation} 
	The objective value of solution $\bm{\delta}_2$ is:
	\begin{equation}\label{lemm}
\begin{array}{l}
{\left\| {{{\bm{\delta }}_2}} \right\|_0} + \frac{\rho }{2}\left\| {{{\bm{\delta }}_2} - {\bm{z}} + {\bm{s}}  + \bm{x_0}} \right\|_2^2 \ge {\left\| {{{\bm{\delta }}_2}} \right\|_0} + \frac{\rho }{2}h({\bm{z}} - {\bm{s}} -\bm{x_0},{k_2})\\
 \quad \quad \quad \quad \quad \quad \quad \quad \;\; \ \  = n - {k_2} + \frac{\rho }{2}h({\bm{z}} - {\bm{s}} -\bm{x_0},{k_2})
\end{array}
	\end{equation} 
The inequality in Eqn. (\ref{lemm}) holds due to Lemma \ref{lemma}. 

	If $k_2 > k_1$, then according to the definition of $\bm{\delta}_1$, we have
	\begin{equation}
	h(\bm{z}-\bm{s}-\bm{x_0}, k_2)-h(\bm{z}-\bm{s}-\bm{x_0}, k_1)>\frac{2}{\rho} (k_2-k_1)
	\end{equation} 
	So that 
	\begin{equation}
	\big(n-k_2 + \frac{\rho }{2}h(\bm{z}-\bm{s}-\bm{x_0}, k_2)\big) - \big(n-k_1 + \frac{\rho }{2}h(\bm{z}-\bm{s}-\bm{x_0}, k_1) \big)>0
	\end{equation} 
	
	If $k_1 > k_2$, then according to the definition of $\bm{\delta}_1$, we have
	\begin{equation}
	h(\bm{z}-\bm{s}-\bm{x_0}, k_1)-h(\bm{z}-\bm{s}-\bm{x_0}, k_2)<\frac{2}{\rho} (k_1-k_2)
	\end{equation} 
	So that 
	\begin{equation}
	\big(n-k_2 + \frac{\rho }{2}h(\bm{z}-\bm{s}-\bm{x_0}, k_2)\big) - \big(n-k_1 + \frac{\rho }{2}h(\bm{z}-\bm{s}-\bm{x_0}, k_1) \big) > 0
	\end{equation} 
	Thus, we can see that our solution $\bm{\delta}_1$ can achieve the minimum objective value and it is the optimal solution.
\end{proof}


When solving the subproblem 
(\ref{l0a}) according to Theorem \ref{theorem}, we enforce a hidden constraint on the distortion $\bm{\delta}=\bm{x}-\bm{x_0}$, that the square of each non-zero element in $\bm{\delta}$ must be larger than $\frac{2}{\rho}$. Therefore, a smaller $\rho$ would push ADMM method to find $\bm{\delta}$ with larger non-zero elements, 
thus reducing the number of non-zero elements and decreasing $L_0$ norm. 
Empirically, we find the constant $\rho$ represents a trade-off between attack success rate and $L_0$ norm of the distortion, i.e., a larger $\rho$ can help find solutions with higher attack success rate at the cost of larger $L_0$ norm of the distortion. 


Then the complete solution to the $L_0$ attack problem using the ADMM framework can be derived similar to the $L_2$ attack. 
More specifically, in each iteration, Theorem \ref{theorem} is applied to obtain the optimal $\bm \delta$ and $\bm{x}$. Then we solve Eqn. (\ref{adm1}) with Adam optimizer and update parameters through (\ref{supdate}).

\subsection{$L_1$ Attack}

For $L_1$ attack, the subproblem (\ref{Sub1}) has the form:
\begin{equation}\label{l1a}
\mathop {\min }\limits_{\bm{x}}  \quad \left\| \bm{x} -\bm{x_0}  \right\|_{1} + \frac{\rho }{2}\left\| {\bm{x}  - \bm{z} + \bm{s}} \right\|_2^2
\end{equation}
Problem (\ref{l1a}) has the closed-form solution. If we change the variable $\bm{\delta} = \bm x - \bm x_0$, then problem becomes
\begin{equation}\label{l1a_2}
\mathop {\min }\limits_{\bm{\delta}}  \quad \left\| \bm{\delta}  \right\|_{1} + \frac{\rho }{2}\left\| {\bm{\delta} + \mathbf x_0  - \bm{z} + \bm{s}} \right\|_2^2.
\end{equation}
The solution of problem \eqref{l1a_2} is given by the soft thresholding operator evaluated at the point $(\mathbf z - \mathbf s - \mathbf x_0 )$ with a parameter $1/\rho$ \cite{parikh2014proximal},
\begin{align}
    \bm{\delta}^* = \left(\mathbf z - \mathbf s- \mathbf x_0 - 1/\rho \right)_+ - \left(- (\mathbf z - \mathbf s- \mathbf x_0) - 1/\rho \right)_+,
\end{align}
where $(\cdot)_+$ is taken in elementwise, and $(x)_+ = x$ if $x \geq 0$, and $0$ otherwise. Therefore, the solution to problem \eqref{l1a} is given by
\begin{align}
    \mathbf x^* = \mathbf x_0 + \bm{\delta}^*.
\end{align}


The complete solution to the $L_1$ attack problem using the ADMM framework is similar to the $L_2$ attack. In each iteration, we obtain the closed-form solution of the first subproblem (\ref{l1a}) and then Adam optimizer is utilized to solve the second subproblem (\ref{adm1}). Next we update the parameters through Eqn. (\ref{supdate}). 

\subsection{$L_\infty$ Attack}

For $L_\infty$ attack, the subproblem (\ref{Sub1}) has the form:
\begin{equation}\label{lia}
\mathop {\min }\limits_{\bm{x}}  \quad \left\| \bm{x}-\bm{x_0}  \right\|_{\infty} + \frac{\rho }{2}\left\| {\bm{x}  - \bm{z} + \bm{s}} \right\|_2^2
\end{equation}
This problem does not have a closed form solution. One possible  method is to derive the KKT conditions of problem \eqref{lia} \cite{parikh2014proximal}. Here we use stochastic gradient decent methods  to solve it. In the experiments, we find that the Adam optimizer \cite{KingmaB2015adam} could achieve fast and robust convergence results. So Adam optimizer is utilized to solve Eqn. (\ref{lia}). Since Eqn. (\ref{lia}) is relatively simpler compared with Eqn. (\ref{adm1}), the complexity for solving Eqn. (\ref{lia}) with Adam optimizer is negligible. 

The complete solution to the $L_\infty$ attack problem using the ADMM framework can be derived similar to the $L_2$ attack. In the $k$-th iteration, we first use Adam optimizer to get the optimal  $\bm{x}^{k+1}$ in  Eq. (\ref{lia}). Then we solve Eq. (\ref{adm1}) and update parameters through Eq. (\ref{supdate}) as the $L_2$ attack.

\section{Performance Evaluation}

\begin{table*}\small
 \centering
  \caption{Adversarial attack success rate (ASR) and distortion of different $L_2$ attacks for different datasets} 
  \label{table_l2_all}
  \scalebox{1}[1]{
   \begin{threeparttable}
\begin{tabular}{c|c|c|c|c|c|c|c|c|c|c|c|c|c}
\toprule[1pt]
\multirow {2}{*}{Data Set}  &  \multirow {2}{*}{Attack Method}  & \multicolumn{4}{c|}{\makecell{Best Case}}   &  \multicolumn{4}{c|}{Average Case}   &  \multicolumn{4}{c}{Worst Case}   \\
  \cline{3-14}  
 &  & \multicolumn{1}{c}{ASR}  & \multicolumn{1}{c}{$\bm{L_2}$} & \multicolumn{1}{c}{$L_1$} & \multicolumn{1}{c|}{$L_\infty$} 
  & \multicolumn{1}{c}{ASR}  & \multicolumn{1}{c}{$\bm{L_2}$} & \multicolumn{1}{c}{$L_1$} & \multicolumn{1}{c|}{$L_\infty$} 
  & \multicolumn{1}{c}{ASR}  & \multicolumn{1}{c}{$\bm{L_2}$} & \multicolumn{1}{c}{$L_1$} & \multicolumn{1}{c}{$L_\infty$}  \\
\midrule[1pt]
\multirow {4}{*}{MNIST}  &  FGM($L_2$) & 99.4  & \textbf{2.245} &  25.84 & 0.574 & 34.6 & \textbf{3.284}   & 39.15 & 0.747 & 0 & \textbf{N.A.} & N.A. & N.A.\\
  \cline{2-14} 
&  IFGM($L_2$) & 100  & \textbf{1.58} &  18.51 & 0.388 & 99.9 & \textbf{2.50}  & 32.63 & 0.562 & 99.6 & \textbf{3.958} & 55.04  & 0.783\\
  \cline{2-14} 
 & C\&W($L_2$) & 100  &\textbf{ 1.393}  & 13.57 & 0.402 & 100 & \textbf{2.002}   & 22.31 & 0.54 & 99.9 & \textbf{2.598}  &  31.43 & 0.689\\
  \cline{2-14} 
&  ADMM($L_2$) & 100  & \textbf{1.288} & 13.87 & 0.345 & 100 & \textbf{1.873} & 22.52 & 0.498 & 100 &  \textbf{2.445}  &31.427 & 0.669\\
\midrule[1pt]
\multirow {4}{*}{CIFAR-10}&  FGM($L_2$) & 99.5  &  \textbf{0.421} & 14.13 & 0.05 & 42.8 & \textbf{1.157} &  39.5  & 0.136 & 0.7 &  \textbf{3.115}& 107.1 & 0.369\\
\cline{2-14}
 & IFGM($L_2$) & 100  & \textbf{0.191}  &6.549  & 0.022 & 100 & \textbf{0.432}  & 15.13  & 0.047 & 100 & \textbf{0.716} & 25.22 & 0.079\\
\cline{2-14}
&  C\&W($L_2$) & 100  & \textbf{0.178} & 6.03  & 0.019 & 100 & \textbf{0.347} &  12.115 & 0.0364 & 99.9 & \textbf{0.481}  & 16.75& 0.0536 \\
\cline{2-14}
&  ADMM($L_2$) & 100  & \textbf{0.173} & 5.8  & 0.0192 & 100 &  \textbf{0.337} &11.65  & 0.0365 & 100 & \textbf{0.476}  &  16.73 & 0.0535\\
\midrule[1pt]
\multirow {4}{*}{ImageNet}& FGM($L_2$) & 12 & \textbf{2.29}  &752.9 & 0.087 & 1 & \textbf{6.823} &  2338 & 0.25 & 0 & \textbf{N.A.} & N.A. & N.A.\\
\cline{2-14}
&  IFGM($L_2$) & 100 & \textbf{1.057} &  349.55 & 0.034 & 100 &  \textbf{2.461} & 823.52 & 0.083 & 98 & \textbf{4.448} & 1478.8  & 0.165\\
\cline{2-14}
&  C\&W($L_2$) & 100 & \textbf{0.48} & 142.4 & 0.016 &  100 &  \textbf{0.681} & 215.4 & 0.03 & 100 &\textbf{0.866}   & 275.4& 0.042\\
\cline{2-14}
&  ADMM($L_2$) & 100  & \textbf{0.416} &  117.3& 0.015 & 100 & \textbf{0.568}  &  177.6 & 0.022 & 97 & \textbf{0.701} & 229.08  & 0.0322\\
\bottomrule[1pt]
  \end{tabular}
\end{threeparttable}}
\end{table*}

The proposed ADMM attacks are compared with state-of-the-art attacks, including C\&W attacks \cite{carlini2017towards}, EAD attack, FGM and IFGM attacks, on three image classification datasets, MNIST \cite{Lecun1998gradient}, CIFAR-10 \cite{Krizhevsky2009learning} and ImageNet \cite{deng2009imagenet}. 
We also test our attacks against two defenses, defensive distillation \cite{papernot2016distillation} and adversarial training \cite{tram2018ensemble}, and evaluate the transferability of ADMM attacks.

\subsection{Experiment Setup and Parameter Setting}
Our experiment setup is based on C\&W attack setup for fair comparisons.
Two networks are trained for MNIST and CIFAR-10 datasets, respectively.   For the ImageNet dataset, a pre-trained network is utilized. 
The network architecture for MNIST and CIFAR-10 has four convolutional layers, two max pooling layers, two fully connected layers and a softmax layer. 
It can achieve 99.5\% accuracy on MNIST and 80\% accuracy on CIFAR-10. 
For ImageNet, a pre-trained Inception v3 network \cite{Szegedy2016RethinkingTI} is applied so there is no need to train our own model. The Google Inception model can achieve 96\% top-5 accuracy with image inputs of size $299\times 299\times 3$.
All experiments are conducted on machines with an Intel I7-7700K CPU, 32 GB RAM and an NVIDIA GTX 1080 TI GPU.

The implementations of FGM and IFGM are based on the CleverHans package \cite{papernot2016cleverhans}. The key distortion parameter $\epsilon $ is determined through a fine-grained grid search. For each image, the smallest $\epsilon $ in the grid leading to a successful attack is reported. 
For IFGM, we perform 10 FGM iterations. The distortion parameter $\epsilon'$ in each FGM iteration is set to be $\epsilon/10 $, which is quite effective shown in \cite{tram2018ensemble}.

The implementations of C\&W attacks and EAD attack are based on the github code released by the authors. The EAD attack has two decision rules when selecting the final adversarial example: the least elastic-net (EN) and $L_1$ distortion measurement ($L_1$). Usually, the $L_1$ decision rule can achieve lower $L_1$ distortion than the EN decision rule as the EN decision rule considers a mixture of $L_1$ and $L_2$ distortions. We use the $L_1$ decision rule for fair comparison.

\subsection{Attack Success Rate and Distortion for ADMM $L_2$ attack}

The ADMM $L_2$ attack is compared with FGM, IFGM and C\&W $L_2$ attacks. 
The attack success rate (ASR) represents the percentage of the constructed adversarial examples that are successfully classified as target labels. The average distortion of all successful adversarial examples is reported. 
For zero ASR, its distortion is not available (N.A.).
We craft adversarial examples on MNIST, CIFAR-10  and ImageNet. For MNIST and CIFAR-10, 1000 correctly classified images are randomly selected from the test sets and 9 target labels are tested for each image, so we perform 9000 attacks for each dataset using each attack method.
For ImageNet, 100 correctly classified images are randomly selected and 9 random target labels are used for each image.

The parameter $\rho$ is fixed to 20. The number of ADMM iterations is set to 10. In each ADMM iteration, Adam optimizer is utilized to solve the second subproblem based on stochastic gradient descent. When using Adam optimizer, we do binary search for 9 steps on the parameter $c$ (starting from 0.001) and runs 1000 learning iterations for each $c$ with learning rate 0.02 for MNIST and 0.002 for CIFAR-10 and ImageNet. The attack transferability parameter is set to $\kappa = 0$.

Table \ref{table_l2_all} shows the results on MNIST, CIFAR-10 and ImageNet. 
As we can see, FGM fails to generate adversarial examples with high success rate since it is designed to be fast, rather than optimal. 
Among IFGM, C\&W and ADMM $L_2$ attacks, ADMM achieves the lowest $L_2$ distortion for the best case, average case and worst case. 
IFGM has larger $L_2$ distortions compared with C\&W and ADMM attacks on the three datasets, especially on ImageNet. 
For MNIST, the ADMM attack can reduce the $L_2$ distortion by about 7\% compared with C\&W $L_2$ attack. This becomes more prominent on ImageNet that ADMM reduces $L_2$ distortion by 19\% comparing with C\&W in the worst case.


We also observe that on CIFAR-10, ADMM $L_2$ attack can achieve lower $L_2$ distortions but the reductions are not as prominent as that on MNIST or ImageNet. 
The reason may be that CIFAR-10 is the easiest dataset to attack since it requires the lowest $L_2$ distortion among the three datasets. 
So both ADMM $L_2$ attack and C\&W $L_2$ attack can achieve quite good performance.
Note that in most cases on the three datasets, ADMM $L_2$ attack can achieve lower $L_1$, $L_2$ and $L_\infty$ distorions than C\&W $L_2$ attack, indicating a comprehensive enhancement of the ADMM $L_2$ attack over C\&W $L_2$ attack.

\begin{table}\small
 \centering
  \caption{Adversarial attack success rate and distortion of ADMM and C\&W $L_0$ attacks for MNIST and CIFAR-10} 
  \label{admm_l0}
  \scalebox{1}[1]{
   \begin{threeparttable}
\begin{tabular}{c|c|c|c|c|c|c|c}
\toprule[1pt]
\multirow{2}{*}{\makecell{Dataset}} & \multirow{2}{*}{\makecell{Attack\\method}}& \multicolumn{2}{c|}{Best case} & \multicolumn{2}{c|}{Average case} & \multicolumn{2}{c}{Worst case} \\
\cline{3-8}
 &  & \multicolumn{1}{c}{ASR}  & \multicolumn{1}{c|}{$L_0$} & \multicolumn{1}{c}{ASR} & \multicolumn{1}{c|}{$L_0$} & \multicolumn{1}{c}{ASR} &  \multicolumn{1}{c}{$L_0$} \\
\midrule[1pt]
\multirow{2}{*}{\makecell{MNIST}} & C\&W($L_0$) & 100 & 8.1  &  100 & 17.48 & 100 & 31.48  \\
\cline{2-8}
 & ADMM($L_0$) & 100 & 8 & 100& 15.71 & 100 & 25.87 \\
\midrule[1pt]
\multirow{2}{*}{\makecell{CIFAR}} & C\&W($L_0$) & 100& 8.6 &100  & 19.6 &100 &  34.4 \\
\cline{2-8}
 & ADMM($L_0$)  & 100 & 8.25 & 100 & 18.8 &100 & 31.2 \\
 \bottomrule[1pt]
  \end{tabular}
\end{threeparttable}}
\end{table}

\begin{table*}\small
 \centering
  \caption{Adversarial attack success rate (ASR) and distortion of different $L_1$ attacks for different datasets} 
  \label{table_l1_all}
  \scalebox{1}[1]{
   \begin{threeparttable}
\begin{tabular}{c|c|c|c|c|c|c|c|c|c|c|c|c|c}
\toprule[1pt]
\multirow {2}{*}{Data Set}  &  \multirow {2}{*}{Attack Method}  & \multicolumn{4}{c|}{\makecell{Best Case}}   &  \multicolumn{4}{c|}{Average Case}   &  \multicolumn{4}{c}{Worst Case}   \\
  \cline{3-14}  
 &  & \multicolumn{1}{c}{ASR}  & \multicolumn{1}{c}{$\bm{L_1}$} & \multicolumn{1}{c}{$L_2$} & \multicolumn{1}{c|}{$L_\infty$} 
  & \multicolumn{1}{c}{ASR}  & \multicolumn{1}{c}{$\bm{L_1}$} & \multicolumn{1}{c}{$L_2$} & \multicolumn{1}{c|}{$L_\infty$} 
  & \multicolumn{1}{c}{ASR}  & \multicolumn{1}{c}{$\bm{L_1}$} & \multicolumn{1}{c}{$L_2$} & \multicolumn{1}{c}{$L_\infty$}  \\
\midrule[1pt]
\multirow {4}{*}{MNIST}  &  FGM($L_1$) & 100 & \textbf{29.6} & 2.42& 0.57& 36.5 & \textbf{51.2} & 3.99 & 0.8 & 0 & \textbf{N.A.} & N.A. & N.A.  \\
  \cline{2-14} 
&  IFGM($L_1$) & 100& \textbf{18.7}&1.6 &0.41 &100 &\textbf{33.9} &2.6 &0.58 &100 &\textbf{54.8} &4.04 &0.81 \\
  \cline{2-14} 
&  EAD($L_1$) & 100 & \textbf{7.08} & 1.49 & 0.56 & 100  & \textbf{12.5} & 2.08 &  0.77& 100 & \textbf{18.8} & 2.57 & 0.92    \\
  \cline{2-14} 
&  ADMM($L_1$)& 100 & \textbf{6.0} & 2.07 & 0.97 & 100 & \textbf{10.61} & 2.72 & 0.99 & 100  & \textbf{16.6} & 3.41 & 1  \\
\midrule[1pt]
\multirow {4}{*}{CIFAR-10}&  FGM($L_1$)& 98.5 & \textbf{18.25} & 0.53 & 0.057 & 47 & \textbf{48.32} & 1.373 & 0.142 & 1  & \textbf{33.99} & 0.956 & 0.101 \\
\cline{2-14}
 & IFGM($L_1$) & 100 & \textbf{6.28} & 0.184& 0.21& 100& \textbf{13.72}&0.394 & 0.44 &100 & \textbf{22.84} &0.65 & 0.74\\
\cline{2-14}
&  EAD($L_1$) & 100 & \textbf{2.44} & 0.31 & 0.084 & 100 & \textbf{6.392} & 0.6 & 0.185 & 100 & \textbf{10.21} & 0.865 & 0.31 \\
  \cline{2-14} 
&  ADMM($L_1$) & 100& \textbf{2.09} & 0.319 & 0.102 & 100  & \textbf{5.0} & 0.591 & 0.182 & 100 & \textbf{7.453} & 0.77 & 0.255  \\
\midrule[1pt]
\multirow {4}{*}{ImageNet}& FGM($L_1$) & 12 & \textbf{229} & 0.73  & 0.028 & 1 & \textbf{67} & 0.165 & 0.08 & 0 & \textbf{N.A.} & N.A. & N.A. \\
\cline{2-14}
&  IFGM($L_1$) &93 & \textbf{311} & 0.966 & 0.033& 67 & \textbf{498.5}& 1.5& 0.051& 47& \textbf{720.2} & 2.2& 0.08 \\
\cline{2-14}
&  EAD($L_1$) & 100 & \textbf{65.4} & 0.632 & 0.047 & 100 & \textbf{165.5} & 1.02 & 0.06 & 100 & \textbf{290} & 1.43 & 0.08  \\
  \cline{2-14} 
&  ADMM($L_1$) & 100 & \textbf{56.1} & 0.904 & 0.053 & 100 & \textbf{92.7} & 1.15 & 0.0784 & 100 & \textbf{142.1} & 1.473 &  0.102  \\
\bottomrule[1pt]
  \end{tabular}
\end{threeparttable}}
\end{table*}


\subsection{Attack Success Rate and Distortion for ADMM $L_0$ attack}
The performance of ADMM $L_0$ attack in terms of attack success rate and $L_0$ norm of distortion is demonstrated in this section. The ADMM $L_0$ attack is compared with C\&W $L_0$ attack on MNIST and CIFAR-10. 500 images are randomly selected from the test sets of MNIST and CIFAR-10, respectively. Each image has 9 target labels and we perform 4500 attacks for each dataset using either ADMM or C\&W $L_0$ attack.  

For ADMM $L_0$ attack, 9 binary search steps are performed to search for the parameter $\rho$ while $c$ is fixed to 20 for MNIST and 200 for CIFAR-10. The initial value of $\rho$ is set to 3 for MNIST and 40 for CIFAR-10, respectively.
The number of ADMM iterations is 10. 
In each ADMM iteration,  Adam optimizer is utilized to solve the second subproblem with 1000 Adam iterations while the learning rate is set to 0.01 for MNIST and CIFAR-10. 

The results of the $L_0$ attacks are shown in Table \ref{admm_l0}. As observed from the table, both C\&W and ADMM $L_0$ attacks can achieve 100\% attack success rate. For the best case, C\&W $L_0$ attack and ADMM $L_0$ attack have relatively close performance in terms of $L_0$ distortion. For the worst case, ADMM $L_0$ attack can achieve lower $L_0$ distortion than C\&W. 
ADMM $L_0$ attack reduces the $L_0$ distortion by up to 17\% on MNIST.
We also note that the differences between C\&W and ADMM $L_0$ attacks are smaller on CIFAR-10 than that on MNIST. 

\begin{table*}\small
 \centering
  \caption{Adversarial attack success rate (ASR) and distortion of different $L_\infty$ attacks for different datasets} 
  \label{table_li_all}
  \scalebox{1}[1]{
   \begin{threeparttable}
\begin{tabular}{c|c|c|c|c|c|c|c|c|c|c|c|c|c}
    \toprule[1pt]
\multirow {2}{*}{Data Set}  &  \multirow {2}{*}{Attack Method}  & \multicolumn{4}{c|}{\makecell{Best Case}}   &  \multicolumn{4}{c|}{Average Case}   &  \multicolumn{4}{c}{Worst Case}   \\
  \cline{3-14}  
 &  & \multicolumn{1}{c}{ASR}  & \multicolumn{1}{c}{$\bm{L_\infty}$} & \multicolumn{1}{c}{$L_1$} & \multicolumn{1}{c|}{$L_2$} 
  & \multicolumn{1}{c}{ASR}  & \multicolumn{1}{c}{$\bm{L_\infty}$} & \multicolumn{1}{c}{$L_1$} & \multicolumn{1}{c|}{$L_2$} 
  & \multicolumn{1}{c}{ASR}  & \multicolumn{1}{c}{$\bm{L_\infty}$} & \multicolumn{1}{c}{$L_1$} & \multicolumn{1}{c}{$L_2$}  \\
\midrule[1pt]
\multirow {3}{*}{MNIST}  &  FGM($L_\infty$) & 100 & \textbf{0.194} & 84.9 & 4.04 & 35 & \textbf{0.283} & 122.7 & 5.85 & 0 & \textbf{N.A.} & N.A. &N.A. \\
  \cline{2-14} 
&  IFGM($L_\infty$) & 100 &\textbf{0.148} & 50.9 & 2.48 &100 & \textbf{0.233} & 71.2&3.44&100&\textbf{0.378} &96.8&4.64 \\
  \cline{2-14} 
&  ADMM($L_\infty$) & 100 & \textbf{0.135} & 35.9 & 2.068 & 100 & \textbf{0.178} & 48 & 2.73 & 100 & \textbf{0.218} & 60.2& 3.37 \\
\midrule[1pt]
\multirow {3}{*}{CIFAR-10}&  FGM($L_\infty$) & 100 &\textbf{0.015} & 42.8& 0.78& 53& \textbf{0.48} &136 &2.5 &1.5 &\textbf{0.31} &712 &14 \\
\cline{2-14}
 & IFGM($L_\infty$) & 100 & \textbf{0.0063} & 14.36 & 0.28 &100 &\textbf{0.015} & 26.2 & 0.54 & 100 & \textbf{0.026} & 37.7 & 0.826 \\
\cline{2-14}
&  ADMM($L_\infty$) & 100& \textbf{0.0061}&12.8 &0.25 &100 &\textbf{0.0114} &23.07 & 0.47&100 &\textbf{0.017} &31.9 &0.65 \\
\midrule[1pt]
\multirow {3}{*}{ImageNet}& FGM($L_\infty$) & 20 & \textbf{0.0873} &22372 & 43.55  & 1.5 & \textbf{0.0005} &134 & 0.26  &0 & \textbf{N.A.} & N.A. & N.A.  \\
\cline{2-14}
&  IFGM($L_\infty$) & 100 & \textbf{0.0046} & 542.4 & 1.27 & 100 & \textbf{0.0128} & 1039.6 & 2.54 & 100& \textbf{0.0253} & 1790.2 & 4.4 \\
\cline{2-14}
&  ADMM($L_\infty$) & 100 & \textbf{0.0041} & 280.2 & 0.773 & 100 & \textbf{0.0059} & 427.7 & 1.10 & 100 & \textbf{0.0092} & 624.1 & 1.6  \\
\bottomrule[1pt]
  \end{tabular}
\end{threeparttable}}
\end{table*}

\subsection{Attack Success Rate and Distortion for ADMM $L_1$ attack}

We compare the ADMM $L_1$ attack with FGM, IFGM and EAD $L_1$ \cite{chen2017ead} attacks. 
The attack success rate (ASR) and the average distortion of all successful adversarial examples are reported. 
We perform the adversarial $L_1$ attacks  on MNIST, CIFAR-10  and ImageNet. For MNIST and CIFAR-10, 1000 correctly classified images are randomly selected from the test sets and 9 target labels are tested for each image, so we perform 9000 attacks for each dataset using each attack method. For ImageNet, 100 correctly classified images and 9 target labels are randomly selected.

The number of ADMM iterations is set to 80. In each ADMM iteration, Adam optimizer is utilized to solve the second subproblem based on stochastic gradient descent. When using Adam optimizer, we run 2000 learning iterations with initial learning rate 0.1 for MNIST and 0.001 for CIFAR-10 and ImageNet. The parameter $c$ is fixed to 2 for MNIST, 40 for CIFAR-10, and 200 for ImageNet. The parameter $\rho$ is fixed to 10 for MNIST, 300 for CIFAR-10, and 2000 for ImageNet. Note that we do not perform binary search of $c$ or $\rho$ as fixed $c$ and $\rho$ can achieve good performance.

The results of the ADMM $L_1$ attack are shown in Table \ref{table_l1_all}. We can observe that both EAD and ADMM $L_1$ attacks can achieve 100\% attack success rate while FGM $L_1$ attack has bad performance and IFGM $L_1$ attack can not guarantee 100\% ASR on ImageNet. ADMM $L_1$ attack can achieve the best performance compared with FGM, IFGM, and EAD $L_1$ attacks. As demonstrated in Table \ref{table_l1_all}, the $L_1$ distortion measurements of ADMM and EAD $L_1$ attacks are relatively close in the best case while the improvement of ADMM $L_1$ attack over EAD $L_1$ attack is much larger for the worst case. In the best case, the ADMM $L_1$ attack can craft adversarial examples with a $L_1$ norm about 14\% smaller than that of the EAD $L_1$ attack on MNIST, CIFAR-10 and ImageNet. For the worst case, the $L_1$ norm of ADMM $L_1$ attack is about 28\% lower on CIFAR-10 and 50\% lower on ImageNet compared with that of EAD $L_1$ attack.

\subsection{Attack Success Rate and Distortion for ADMM $L_\infty$ attack}

The ADMM $L_\infty$ attack is compared with FGM and IFGM $L_\infty$ attacks. 
The attack success rate (ASR) and the average distortion of all successful adversarial examples are reported. 
We perform the adversarial $L_\infty$ attacks  on MNIST, CIFAR-10  and ImageNet. For MNIST and CIFAR-10, 1000 correctly classified images are randomly selected from the test sets and 9 target labels are tested for each image, so we perform 9000 attacks for each dataset using each attack method. For ImageNet, 100 correctly classified images and 9 target labels are randomly selected.

The parameter $\rho$ is fixed to 0.1. The number of ADMM iterations is 100 and the batch size is 90. In each ADMM iteration, Adam optimizer is utilized to solve the first and second subproblem based on stochastic gradient descent. Adam optimizer runs 1000 iterations to get the solution of the first subproblem while it executes 2000 iterations to solve the second subproblem. Note that in the second subproblem, $c$ is fixed to 0.1 as we find fixed $c$ can achieve good performance and there is no need to perform binary search of $c$.
The initial learning rate is set to 0.001 for MNIST and 0.002 for CIFAR-10 and ImageNet. 
The attack transferability parameter is set to $\kappa = 0$ if we do not perform the transferability evaluation. 

The results of the ADMM $L_\infty$ attack are demonstrated in Table \ref{table_li_all}. We can observe that both IFGM and ADMM $L_\infty$ attacks can achieve 100\% attack success rate while FGM has bad performance. ADMM $L_\infty$ attack can achieve the best performance compared with FGM and IFGM $L_\infty$ attacks. We also note that the $L_\infty$ norms of ADMM and IFGM $L_\infty$ attacks are relatively close in the best case. Usually the $L_\infty$ distortion measure of ADMM attack is  smaller than that of IFGM attack by no larger than 10\% for the best case. In the worst case, the improvement of ADMM $L_\infty$ attack over IFGM $L_\infty$ attack is much more obvious. The $L_\infty$ distortion measure of ADMM attack is about 40\% smaller than that of IFGM attack on MNIST or CIFAR-10 dataset for the worst case. On ImageNet, the $L_\infty$ norm of ADMM attack is 64\% lower than that of IFGM attack.

\subsection{ADMM Attack Against Defensive Distillation and Adversarial Training}
ADMM attacks can break the undefended DNNs with high success rate. It is also able to break DNNs with defensive distillation. 
We perform C\&W $L_2$ attack, ADMM $L_0$, $L_1$, $L_2$ and $L_\infty$ attack for different temperature parameters on MNIST and CIFAR-10. 500 randomly selected images are used as source to generate 4500 adversarial examples with 9 targets for each image on MNIST or CIFAR-10. 
We find that the attack success rates of C\&W $L_2$ attack and ADMM four attacks for different temperature $T$ are all 100\%. Since distillation at temperature $T$ causes the value of logits to be approximately $T$ times larger while the relative values of logits remain unchanged, C\&W attack and ADMM attack which work on the relative values of logits do not fail.


We further test ADMM attack against adversarial training on MNIST.
C\&W $L_2$ attack and ADMM $L_2$ attack are utilized to separately generate 9000 adversarial examples with 1000 randomly selected images from the training set as sources. Then we add the adversarial examples with correct labels into the training dataset and retrain the network with the enlarged training dataset.
With the retained network, we perform ADMM attack on the adversarially trained networks (one with C\&W adversarial examples, and one with ADMM adversarial examples), as shown in Fig. \ref{advertraining}.
ADMM $L_2$ attack can break all three networks (one unprotected, one retained with C\&W adversarial examples, and one retained with ADMM adversarial examples) with 100\% success rate.
$L_2$ distortions on the latter two networks are higher than that on the first network, showing some defense effect of adversarial training. 
We also note that $L_2$ distortion on the third network is higher than the second network, which demonstrates higher defense efficiency of performing adversarial training with ADMM adversarial examples (partly because ADMM attack is stronger).

\begin{figure}[tbp]
	\center
	\includegraphics[ scale=0.4]{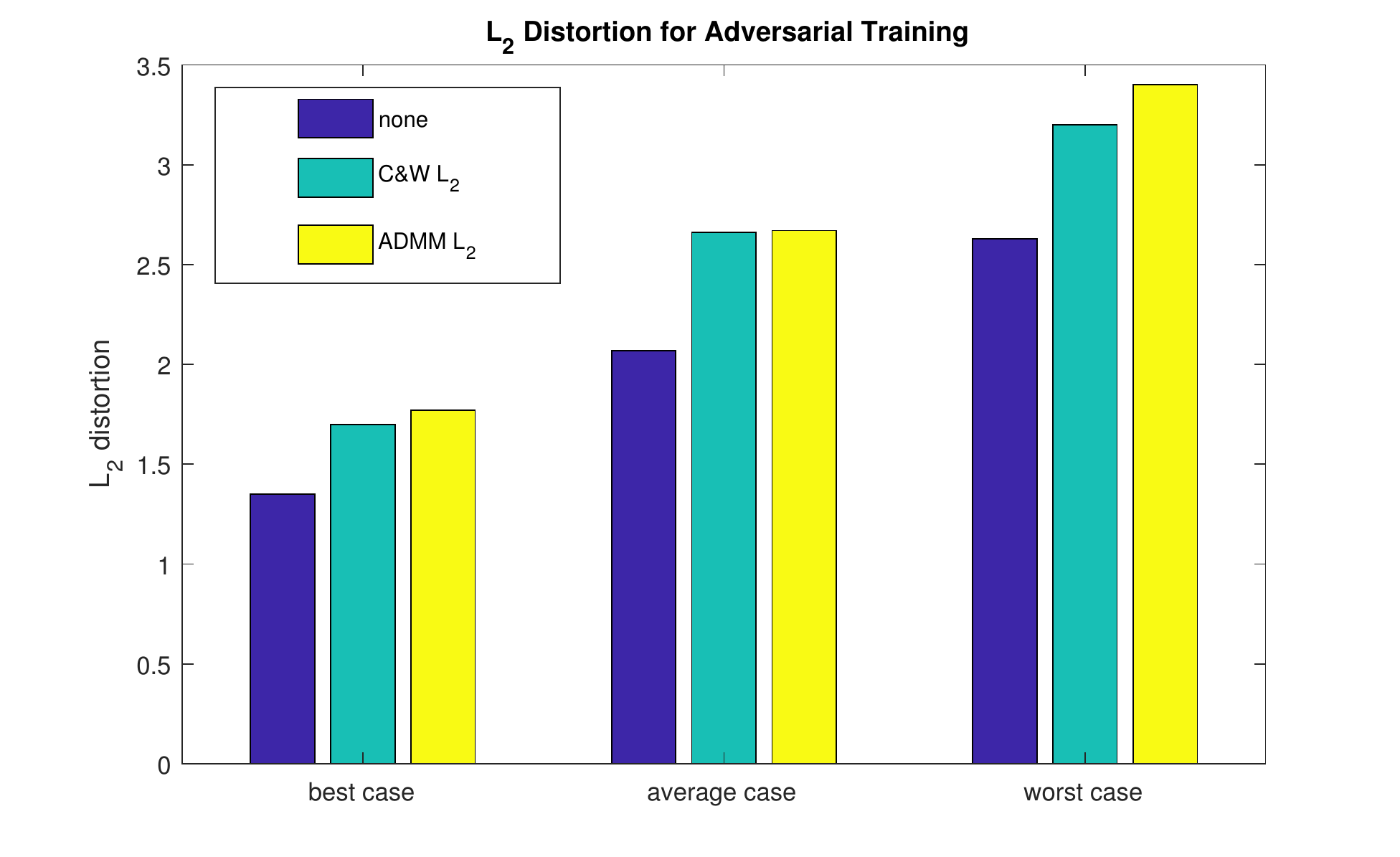}
	\caption{$L_2$ distortion of adversarial training for three cases on MNIST}
	\label{advertraining}
\end{figure}


\subsection{Attack Transferability}
Here we test the transferability of ADMM adversarial attack.
For each value of confidence parameter $\kappa$, we use ADMM $L_2$ attack and C\&W $L_2$ attack to generate 9000 adversarial examples on MNIST, respectively. 
Then these examples are applied to attack the defensively distilled network with temperature $T=100$. The ASR is reported in Fig. \ref{trans}. 
As demonstrated in Fig. \ref{trans}, when $\kappa$ is small, ADMM $L_2$ attack can hardly achieve success on the defensively distilled network, which means the generated adversarial examples are not strong enough to break the defended network. Low transferability of the generated adversarial examples is observed when  $\kappa$ is low. 
As $\kappa$ increases, the ASRs of the three cases increase, demonstrating increasing transferability. 
When $\kappa = 50$, the ASRs of three cases can achieve the maximum value. The  ASR of average case is nearly 98\%, meaning most of the generated adversarial examples on the undefended network can also break the defensively distilled network with $T = 100$. 
Also note that when $\kappa>50$, the ASRs of average case and worst case decrease as $\kappa$ increases. The reason is that it's quite difficult to generate adversarial examples even for the undefended network when $\kappa$ is very large. 
Thus an decrease on the ASR is observed for average case and worst case, and the advantages of strong transferable adversarial examples are mitigated by the difficulty to generate such strong attacks. 
We also note that when $\kappa > 40$, the ASRs of ADMM $L_2$ attack for average case and worst case are higher than the ASRs of C\&W $L_2$ attack, demonstrating higher transferability of the ADMM attack.

\begin{figure}[tbp]
	\center
	\includegraphics[ scale=0.36]{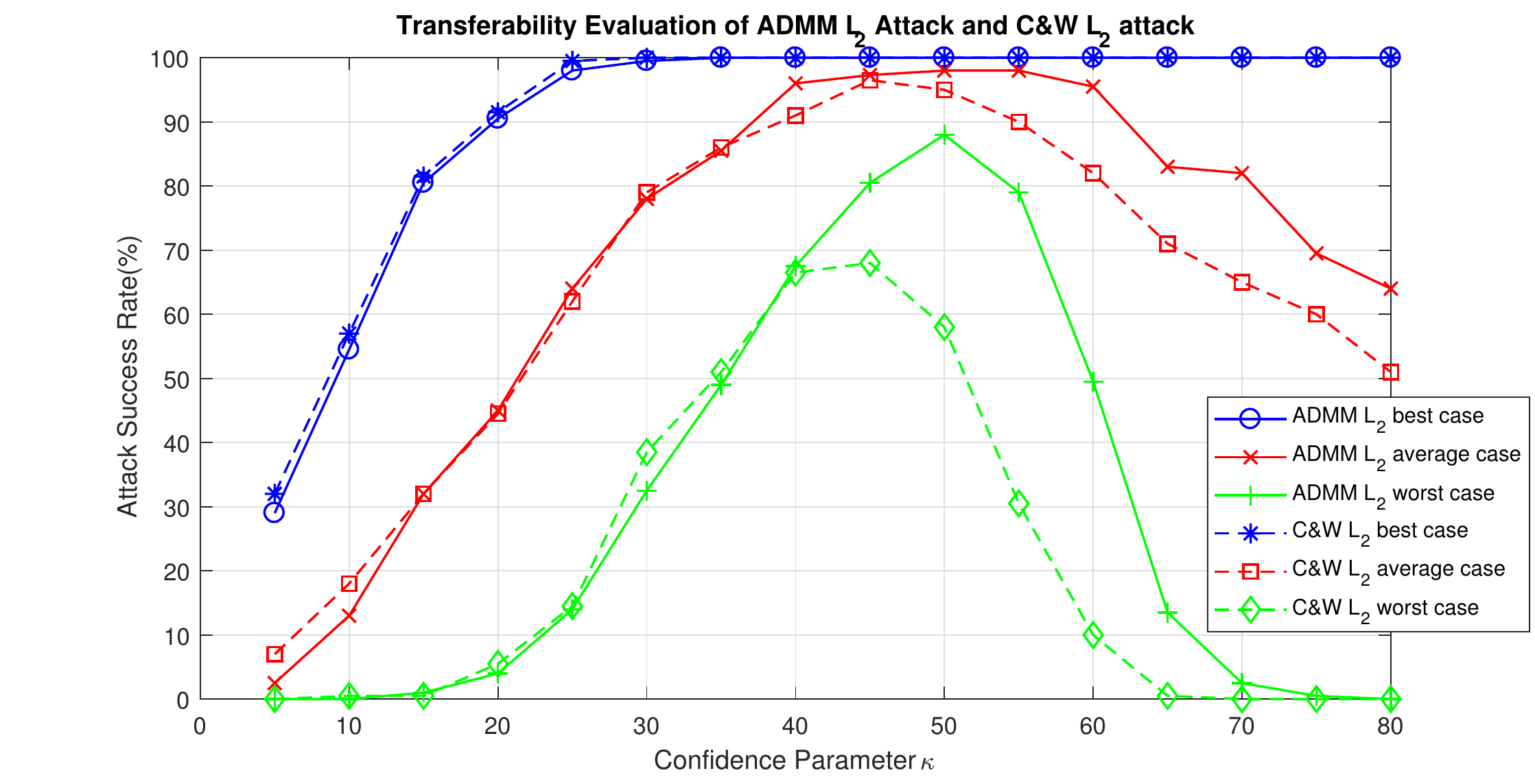}
	\caption{transferiablity evaluation of C\&W and ADMM $L_2$ attacks on MNIST}
	\label{trans}
\end{figure}

\section{Conclusion}

In this paper, we propose an ADMM-based general framework for adversarial attacks. Under the ADMM framework, $L_0$, $L_1$, $L_2$ and $L_\infty$ attacks are proposed and implemented. We compare the ADMM attacks with state-of-the-art adversarial attacks, showing ADMM attacks are so far the strongest. The ADMM attack is also applied to break two defense methods, the defensive distillation and adversarial training. Experimental results show the effectiveness of the proposed ADMM attacks with strong transferability.

%% file: main.bbl

\begin{thebibliography}{00}


\ifx \showCODEN    \undefined \def \showCODEN     #1{\unskip}     \fi
\ifx \showDOI      \undefined \def \showDOI       #1{#1}\fi
\ifx \showISBNx    \undefined \def \showISBNx     #1{\unskip}     \fi
\ifx \showISBNxiii \undefined \def \showISBNxiii  #1{\unskip}     \fi
\ifx \showISSN     \undefined \def \showISSN      #1{\unskip}     \fi
\ifx \showLCCN     \undefined \def \showLCCN      #1{\unskip}     \fi
\ifx \shownote     \undefined \def \shownote      #1{#1}          \fi
\ifx \showarticletitle \undefined \def \showarticletitle #1{#1}   \fi
\ifx \showURL      \undefined \def \showURL       {\relax}        \fi
\providecommand\bibfield[2]{#2}
\providecommand\bibinfo[2]{#2}
\providecommand\natexlab[1]{#1}
\providecommand\showeprint[2][]{arXiv:#2}

\bibitem[\protect\citeauthoryear{Athalye, Carlini, and Wagner}{Athalye
  et~al\mbox{.}}{2018}]%
        {athalye2018obfuscated}
\bibfield{author}{\bibinfo{person}{Anish Athalye}, \bibinfo{person}{Nicholas
  Carlini}, {and} \bibinfo{person}{David Wagner}.}
  \bibinfo{year}{2018}\natexlab{}.
\newblock \showarticletitle{Obfuscated gradients give a false sense of
  security: Circumventing defenses to adversarial examples}.
\newblock \bibinfo{journal}{{\em arXiv preprint arXiv:1802.00420\/}}
  (\bibinfo{year}{2018}).
\newblock


\bibitem[\protect\citeauthoryear{Bhagoji, Cullina, and Mittal}{Bhagoji
  et~al\mbox{.}}{2017}]%
        {bhagoji2017dimensionality}
\bibfield{author}{\bibinfo{person}{Arjun~Nitin Bhagoji},
  \bibinfo{person}{Daniel Cullina}, {and} \bibinfo{person}{Prateek Mittal}.}
  \bibinfo{year}{2017}\natexlab{}.
\newblock \showarticletitle{Dimensionality reduction as a defense against
  evasion attacks on machine learning classifiers}.
\newblock \bibinfo{journal}{{\em arXiv preprint arXiv:1704.02654\/}}
  (\bibinfo{year}{2017}).
\newblock


\bibitem[\protect\citeauthoryear{Boyd, Parikh, Chu, Peleato, Eckstein,
  et~al\mbox{.}}{Boyd et~al\mbox{.}}{2011}]%
        {boyd2011distributed}
\bibfield{author}{\bibinfo{person}{Stephen Boyd}, \bibinfo{person}{Neal
  Parikh}, \bibinfo{person}{Eric Chu}, \bibinfo{person}{Borja Peleato},
  \bibinfo{person}{Jonathan Eckstein}, {et~al\mbox{.}}}
  \bibinfo{year}{2011}\natexlab{}.
\newblock \showarticletitle{Distributed optimization and statistical learning
  via the alternating direction method of multipliers}.
\newblock \bibinfo{journal}{{\em Foundations and Trends{\textregistered} in
  Machine Learning\/}} \bibinfo{volume}{3}, \bibinfo{number}{1}
  (\bibinfo{year}{2011}), \bibinfo{pages}{1--122}.
\newblock


\bibitem[\protect\citeauthoryear{Carlini, Mishra, Vaidya, Zhang, Sherr,
  Shields, Wagner, and Zhou}{Carlini et~al\mbox{.}}{2016}]%
        {carlini2016hidden}
\bibfield{author}{\bibinfo{person}{Nicholas Carlini}, \bibinfo{person}{Pratyush
  Mishra}, \bibinfo{person}{Tavish Vaidya}, \bibinfo{person}{Yuankai Zhang},
  \bibinfo{person}{Micah Sherr}, \bibinfo{person}{Clay Shields},
  \bibinfo{person}{David Wagner}, {and} \bibinfo{person}{Wenchao Zhou}.}
  \bibinfo{year}{2016}\natexlab{}.
\newblock \showarticletitle{Hidden Voice Commands.}. In
  \bibinfo{booktitle}{{\em USENIX Security Symposium}}.
  \bibinfo{pages}{513--530}.
\newblock


\bibitem[\protect\citeauthoryear{Carlini and Wagner}{Carlini and
  Wagner}{2017}]%
        {carlini2017towards}
\bibfield{author}{\bibinfo{person}{Nicholas Carlini} {and}
  \bibinfo{person}{David Wagner}.} \bibinfo{year}{2017}\natexlab{}.
\newblock \showarticletitle{Towards evaluating the robustness of neural
  networks}. In \bibinfo{booktitle}{{\em Security and Privacy (SP), 2017 IEEE
  Symposium on}}. IEEE, \bibinfo{pages}{39--57}.
\newblock


\bibitem[\protect\citeauthoryear{Chen, Sharma, Zhang, Yi, and Hsieh}{Chen
  et~al\mbox{.}}{2017}]%
        {chen2017ead}
\bibfield{author}{\bibinfo{person}{Pin-Yu Chen}, \bibinfo{person}{Yash Sharma},
  \bibinfo{person}{Huan Zhang}, \bibinfo{person}{Jinfeng Yi}, {and}
  \bibinfo{person}{Cho-Jui Hsieh}.} \bibinfo{year}{2017}\natexlab{}.
\newblock \showarticletitle{EAD: elastic-net attacks to deep neural networks
  via adversarial examples}.
\newblock \bibinfo{journal}{{\em arXiv preprint arXiv:1709.04114\/}}
  (\bibinfo{year}{2017}).
\newblock


\bibitem[\protect\citeauthoryear{Deng, Dong, Socher, Li, Li, and Fei-Fei}{Deng
  et~al\mbox{.}}{2009}]%
        {deng2009imagenet}
\bibfield{author}{\bibinfo{person}{Jia Deng}, \bibinfo{person}{Wei Dong},
  \bibinfo{person}{Richard Socher}, \bibinfo{person}{Li-Jia Li},
  \bibinfo{person}{Kai Li}, {and} \bibinfo{person}{Li Fei-Fei}.}
  \bibinfo{year}{2009}\natexlab{}.
\newblock \showarticletitle{Imagenet: A large-scale hierarchical image
  database}. In \bibinfo{booktitle}{{\em Computer Vision and Pattern
  Recognition, 2009. CVPR 2009. IEEE Conference on}}. IEEE,
  \bibinfo{pages}{248--255}.
\newblock


\bibitem[\protect\citeauthoryear{Dhillon, Azizzadenesheli, Lipton, Bernstein,
  Kossaifi, Khanna, and Anandkumar}{Dhillon et~al\mbox{.}}{2018}]%
        {dhillon2018stochastic}
\bibfield{author}{\bibinfo{person}{Guneet~S Dhillon}, \bibinfo{person}{Kamyar
  Azizzadenesheli}, \bibinfo{person}{Zachary~C Lipton}, \bibinfo{person}{Jeremy
  Bernstein}, \bibinfo{person}{Jean Kossaifi}, \bibinfo{person}{Aran Khanna},
  {and} \bibinfo{person}{Anima Anandkumar}.} \bibinfo{year}{2018}\natexlab{}.
\newblock \showarticletitle{Stochastic Activation Pruning for Robust
  Adversarial Defense}.
\newblock \bibinfo{journal}{{\em arXiv preprint arXiv:1803.01442\/}}
  (\bibinfo{year}{2018}).
\newblock


\bibitem[\protect\citeauthoryear{Dziugaite, Ghahramani, and Roy}{Dziugaite
  et~al\mbox{.}}{2016}]%
        {dziugaite2016study}
\bibfield{author}{\bibinfo{person}{Gintare~Karolina Dziugaite},
  \bibinfo{person}{Zoubin Ghahramani}, {and} \bibinfo{person}{Daniel~M Roy}.}
  \bibinfo{year}{2016}\natexlab{}.
\newblock \showarticletitle{A study of the effect of jpg compression on
  adversarial images}.
\newblock \bibinfo{journal}{{\em arXiv preprint arXiv:1608.00853\/}}
  (\bibinfo{year}{2016}).
\newblock


\bibitem[\protect\citeauthoryear{Feinman, Curtin, Shintre, and Gardner}{Feinman
  et~al\mbox{.}}{2017}]%
        {feinman2017detecting}
\bibfield{author}{\bibinfo{person}{Reuben Feinman}, \bibinfo{person}{Ryan~R
  Curtin}, \bibinfo{person}{Saurabh Shintre}, {and} \bibinfo{person}{Andrew~B
  Gardner}.} \bibinfo{year}{2017}\natexlab{}.
\newblock \showarticletitle{Detecting adversarial samples from artifacts}.
\newblock \bibinfo{journal}{{\em arXiv preprint arXiv:1703.00410\/}}
  (\bibinfo{year}{2017}).
\newblock


\bibitem[\protect\citeauthoryear{Goodfellow, Shlens, and Szegedy}{Goodfellow
  et~al\mbox{.}}{2014}]%
        {goodfellow2014explaining}
\bibfield{author}{\bibinfo{person}{Ian~J Goodfellow}, \bibinfo{person}{Jonathon
  Shlens}, {and} \bibinfo{person}{Christian Szegedy}.}
  \bibinfo{year}{2014}\natexlab{}.
\newblock \showarticletitle{Explaining and harnessing adversarial examples}.
\newblock \bibinfo{journal}{{\em arXiv preprint arXiv:1412.6572\/}}
  (\bibinfo{year}{2014}).
\newblock


\bibitem[\protect\citeauthoryear{Guo, Rana, Ciss{\'e}, and van~der Maaten}{Guo
  et~al\mbox{.}}{2017}]%
        {guo2017countering}
\bibfield{author}{\bibinfo{person}{Chuan Guo}, \bibinfo{person}{Mayank Rana},
  \bibinfo{person}{Moustapha Ciss{\'e}}, {and} \bibinfo{person}{Laurens van~der
  Maaten}.} \bibinfo{year}{2017}\natexlab{}.
\newblock \showarticletitle{Countering Adversarial Images using Input
  Transformations}.
\newblock \bibinfo{journal}{{\em arXiv preprint arXiv:1711.00117\/}}
  (\bibinfo{year}{2017}).
\newblock


\bibitem[\protect\citeauthoryear{He, Zhang, Ren, and Sun}{He
  et~al\mbox{.}}{2016}]%
        {he2016deep}
\bibfield{author}{\bibinfo{person}{Kaiming He}, \bibinfo{person}{Xiangyu
  Zhang}, \bibinfo{person}{Shaoqing Ren}, {and} \bibinfo{person}{Jian Sun}.}
  \bibinfo{year}{2016}\natexlab{}.
\newblock \showarticletitle{Deep residual learning for image recognition}. In
  \bibinfo{booktitle}{{\em Proceedings of the IEEE conference on computer
  vision and pattern recognition}}. \bibinfo{pages}{770--778}.
\newblock


\bibitem[\protect\citeauthoryear{Hinton, Deng, Yu, Dahl, Mohamed, Jaitly,
  Senior, Vanhoucke, Nguyen, Sainath, et~al\mbox{.}}{Hinton
  et~al\mbox{.}}{2012}]%
        {hinton2012deep}
\bibfield{author}{\bibinfo{person}{Geoffrey Hinton}, \bibinfo{person}{Li Deng},
  \bibinfo{person}{Dong Yu}, \bibinfo{person}{George~E Dahl},
  \bibinfo{person}{Abdel-rahman Mohamed}, \bibinfo{person}{Navdeep Jaitly},
  \bibinfo{person}{Andrew Senior}, \bibinfo{person}{Vincent Vanhoucke},
  \bibinfo{person}{Patrick Nguyen}, \bibinfo{person}{Tara~N Sainath},
  {et~al\mbox{.}}} \bibinfo{year}{2012}\natexlab{}.
\newblock \showarticletitle{Deep neural networks for acoustic modeling in
  speech recognition: The shared views of four research groups}.
\newblock \bibinfo{journal}{{\em IEEE Signal Processing Magazine\/}}
  \bibinfo{volume}{29}, \bibinfo{number}{6} (\bibinfo{year}{2012}),
  \bibinfo{pages}{82--97}.
\newblock


\bibitem[\protect\citeauthoryear{Hong and Luo}{Hong and Luo}{2017}]%
        {hong2017linear}
\bibfield{author}{\bibinfo{person}{Mingyi Hong} {and} \bibinfo{person}{Zhi-Quan
  Luo}.} \bibinfo{year}{2017}\natexlab{}.
\newblock \showarticletitle{On the linear convergence of the alternating
  direction method of multipliers}.
\newblock \bibinfo{journal}{{\em Mathematical Programming\/}}
  \bibinfo{volume}{162}, \bibinfo{number}{1} (\bibinfo{date}{01 Mar}
  \bibinfo{year}{2017}), \bibinfo{pages}{165--199}.
\newblock
\showISSN{1436-4646}
\showDOI{%
\url{https://doi.org/10.1007/s10107-016-1034-2}}


\bibitem[\protect\citeauthoryear{Kingma and Ba}{Kingma and Ba}{2015}]%
        {KingmaB2015adam}
\bibfield{author}{\bibinfo{person}{Diederik~P. Kingma} {and}
  \bibinfo{person}{Jimmy Ba}.} \bibinfo{year}{2015}\natexlab{}.
\newblock \showarticletitle{Adam: {A} Method for Stochastic Optimization}.
\newblock \bibinfo{journal}{{\em 2015 ICLR\/}}  \bibinfo{volume}{arXiv preprint
  arXiv:1412.6980} (\bibinfo{year}{2015}).
\newblock
\showeprint[arxiv]{1412.6980}
\showURL{%
\url{http://arxiv.org/abs/1412.6980}}


\bibitem[\protect\citeauthoryear{Krizhevsky and Hinton}{Krizhevsky and
  Hinton}{2009}]%
        {Krizhevsky2009learning}
\bibfield{author}{\bibinfo{person}{A. Krizhevsky} {and} \bibinfo{person}{G.
  Hinton}.} \bibinfo{year}{2009}\natexlab{}.
\newblock \showarticletitle{Learning multiple layers of features from tiny
  images}.
\newblock \bibinfo{journal}{{\em Master's thesis, Department of Computer
  Science, University of Toronto\/}} (\bibinfo{year}{2009}).
\newblock


\bibitem[\protect\citeauthoryear{Krizhevsky, Sutskever, and Hinton}{Krizhevsky
  et~al\mbox{.}}{2012}]%
        {krizhevsky2012imagenet}
\bibfield{author}{\bibinfo{person}{Alex Krizhevsky}, \bibinfo{person}{Ilya
  Sutskever}, {and} \bibinfo{person}{Geoffrey~E Hinton}.}
  \bibinfo{year}{2012}\natexlab{}.
\newblock \showarticletitle{Imagenet classification with deep convolutional
  neural networks}. In \bibinfo{booktitle}{{\em Advances in neural information
  processing systems}}. \bibinfo{pages}{1097--1105}.
\newblock


\bibitem[\protect\citeauthoryear{Kurakin, Goodfellow, and Bengio}{Kurakin
  et~al\mbox{.}}{2016}]%
        {kurakin2016adversarial}
\bibfield{author}{\bibinfo{person}{Alexey Kurakin}, \bibinfo{person}{Ian
  Goodfellow}, {and} \bibinfo{person}{Samy Bengio}.}
  \bibinfo{year}{2016}\natexlab{}.
\newblock \showarticletitle{Adversarial examples in the physical world}.
\newblock \bibinfo{journal}{{\em arXiv preprint arXiv:1607.02533\/}}
  (\bibinfo{year}{2016}).
\newblock


\bibitem[\protect\citeauthoryear{Lecun, Bottou, Bengio, and Haffner}{Lecun
  et~al\mbox{.}}{1998}]%
        {Lecun1998gradient}
\bibfield{author}{\bibinfo{person}{Y. Lecun}, \bibinfo{person}{L. Bottou},
  \bibinfo{person}{Y. Bengio}, {and} \bibinfo{person}{P. Haffner}.}
  \bibinfo{year}{1998}\natexlab{}.
\newblock \showarticletitle{Gradient-based learning applied to document
  recognition}.
\newblock \bibinfo{journal}{{\it Proc. IEEE}} \bibinfo{volume}{86},
  \bibinfo{number}{11} (\bibinfo{date}{Nov} \bibinfo{year}{1998}),
  \bibinfo{pages}{2278--2324}.
\newblock
\showISSN{0018-9219}
\showDOI{%
\url{https://doi.org/10.1109/5.726791}}


\bibitem[\protect\citeauthoryear{Makantasis, Karantzalos, Doulamis, and
  Doulamis}{Makantasis et~al\mbox{.}}{2015}]%
        {makantasis2015deep}
\bibfield{author}{\bibinfo{person}{Konstantinos Makantasis},
  \bibinfo{person}{Konstantinos Karantzalos}, \bibinfo{person}{Anastasios
  Doulamis}, {and} \bibinfo{person}{Nikolaos Doulamis}.}
  \bibinfo{year}{2015}\natexlab{}.
\newblock \showarticletitle{Deep supervised learning for hyperspectral data
  classification through convolutional neural networks}. In
  \bibinfo{booktitle}{{\em Geoscience and Remote Sensing Symposium (IGARSS),
  2015 IEEE International}}. IEEE, \bibinfo{pages}{4959--4962}.
\newblock


\bibitem[\protect\citeauthoryear{Nguyen, Yosinski, and Clune}{Nguyen
  et~al\mbox{.}}{2015}]%
        {nguyen2015deep}
\bibfield{author}{\bibinfo{person}{Anh Nguyen}, \bibinfo{person}{Jason
  Yosinski}, {and} \bibinfo{person}{Jeff Clune}.}
  \bibinfo{year}{2015}\natexlab{}.
\newblock \showarticletitle{Deep neural networks are easily fooled: High
  confidence predictions for unrecognizable images}. In
  \bibinfo{booktitle}{{\em Proceedings of the IEEE Conference on Computer
  Vision and Pattern Recognition}}. \bibinfo{pages}{427--436}.
\newblock


\bibitem[\protect\citeauthoryear{Papernot, Goodfellow, Sheatsley, Feinman, and
  McDaniel}{Papernot et~al\mbox{.}}{2016a}]%
        {papernot2016cleverhans}
\bibfield{author}{\bibinfo{person}{Nicolas Papernot}, \bibinfo{person}{Ian
  Goodfellow}, \bibinfo{person}{Ryan Sheatsley}, \bibinfo{person}{Reuben
  Feinman}, {and} \bibinfo{person}{Patrick McDaniel}.}
  \bibinfo{year}{2016}\natexlab{a}.
\newblock \showarticletitle{cleverhans v1.0.0: an adversarial machine learning
  library}.
\newblock \bibinfo{journal}{{\em arXiv preprint arXiv:1610.00768\/}}
  (\bibinfo{year}{2016}).
\newblock


\bibitem[\protect\citeauthoryear{Papernot, McDaniel, Jha, Fredrikson, Celik,
  and Swami}{Papernot et~al\mbox{.}}{2016b}]%
        {papernot2016limitations}
\bibfield{author}{\bibinfo{person}{Nicolas Papernot}, \bibinfo{person}{Patrick
  McDaniel}, \bibinfo{person}{Somesh Jha}, \bibinfo{person}{Matt Fredrikson},
  \bibinfo{person}{Z~Berkay Celik}, {and} \bibinfo{person}{Ananthram Swami}.}
  \bibinfo{year}{2016}\natexlab{b}.
\newblock \showarticletitle{The limitations of deep learning in adversarial
  settings}. In \bibinfo{booktitle}{{\em Security and Privacy (EuroS\&P), 2016
  IEEE European Symposium on}}. IEEE, \bibinfo{pages}{372--387}.
\newblock


\bibitem[\protect\citeauthoryear{Papernot, McDaniel, Wu, Jha, and
  Swami}{Papernot et~al\mbox{.}}{2016c}]%
        {papernot2016distillation}
\bibfield{author}{\bibinfo{person}{Nicolas Papernot}, \bibinfo{person}{Patrick
  McDaniel}, \bibinfo{person}{Xi Wu}, \bibinfo{person}{Somesh Jha}, {and}
  \bibinfo{person}{Ananthram Swami}.} \bibinfo{year}{2016}\natexlab{c}.
\newblock \showarticletitle{Distillation as a defense to adversarial
  perturbations against deep neural networks}. In \bibinfo{booktitle}{{\em
  Security and Privacy (SP), 2016 IEEE Symposium on}}. IEEE,
  \bibinfo{pages}{582--597}.
\newblock


\bibitem[\protect\citeauthoryear{Parikh, Boyd, et~al\mbox{.}}{Parikh
  et~al\mbox{.}}{2014}]%
        {parikh2014proximal}
\bibfield{author}{\bibinfo{person}{Neal Parikh}, \bibinfo{person}{Stephen
  Boyd}, {et~al\mbox{.}}} \bibinfo{year}{2014}\natexlab{}.
\newblock \showarticletitle{Proximal algorithms}.
\newblock \bibinfo{journal}{{\em Foundations and Trends{\textregistered} in
  Optimization\/}} \bibinfo{volume}{1}, \bibinfo{number}{3}
  (\bibinfo{year}{2014}), \bibinfo{pages}{127--239}.
\newblock


\bibitem[\protect\citeauthoryear{Silver, Huang, Maddison, Guez, Sifre, Van
  Den~Driessche, Schrittwieser, Antonoglou, Panneershelvam, Lanctot,
  et~al\mbox{.}}{Silver et~al\mbox{.}}{2016}]%
        {silver2016mastering}
\bibfield{author}{\bibinfo{person}{David Silver}, \bibinfo{person}{Aja Huang},
  \bibinfo{person}{Chris~J Maddison}, \bibinfo{person}{Arthur Guez},
  \bibinfo{person}{Laurent Sifre}, \bibinfo{person}{George Van Den~Driessche},
  \bibinfo{person}{Julian Schrittwieser}, \bibinfo{person}{Ioannis Antonoglou},
  \bibinfo{person}{Veda Panneershelvam}, \bibinfo{person}{Marc Lanctot},
  {et~al\mbox{.}}} \bibinfo{year}{2016}\natexlab{}.
\newblock \showarticletitle{Mastering the game of Go with deep neural networks
  and tree search}.
\newblock \bibinfo{journal}{{\em nature\/}} \bibinfo{volume}{529},
  \bibinfo{number}{7587} (\bibinfo{year}{2016}), \bibinfo{pages}{484--489}.
\newblock


\bibitem[\protect\citeauthoryear{Su, Vargas, and Kouichi}{Su
  et~al\mbox{.}}{2017}]%
        {su2017one}
\bibfield{author}{\bibinfo{person}{Jiawei Su},
  \bibinfo{person}{Danilo~Vasconcellos Vargas}, {and} \bibinfo{person}{Sakurai
  Kouichi}.} \bibinfo{year}{2017}\natexlab{}.
\newblock \showarticletitle{One pixel attack for fooling deep neural networks}.
\newblock \bibinfo{journal}{{\em arXiv preprint arXiv:1710.08864\/}}
  (\bibinfo{year}{2017}).
\newblock


\bibitem[\protect\citeauthoryear{Szegedy, Vanhoucke, Ioffe, Shlens, and
  Wojna}{Szegedy et~al\mbox{.}}{2016}]%
        {Szegedy2016RethinkingTI}
\bibfield{author}{\bibinfo{person}{Christian Szegedy}, \bibinfo{person}{Vincent
  Vanhoucke}, \bibinfo{person}{Sergey Ioffe}, \bibinfo{person}{Jonathon
  Shlens}, {and} \bibinfo{person}{Zbigniew Wojna}.}
  \bibinfo{year}{2016}\natexlab{}.
\newblock \showarticletitle{Rethinking the Inception Architecture for Computer
  Vision}.
\newblock \bibinfo{journal}{{\em 2016 IEEE Conference on Computer Vision and
  Pattern Recognition (CVPR)\/}} (\bibinfo{year}{2016}),
  \bibinfo{pages}{2818--2826}.
\newblock


\bibitem[\protect\citeauthoryear{Szegedy, Zaremba, Sutskever, Bruna, Erhan,
  Goodfellow, and Fergus}{Szegedy et~al\mbox{.}}{2013}]%
        {szegedy2013intriguing}
\bibfield{author}{\bibinfo{person}{Christian Szegedy},
  \bibinfo{person}{Wojciech Zaremba}, \bibinfo{person}{Ilya Sutskever},
  \bibinfo{person}{Joan Bruna}, \bibinfo{person}{Dumitru Erhan},
  \bibinfo{person}{Ian Goodfellow}, {and} \bibinfo{person}{Rob Fergus}.}
  \bibinfo{year}{2013}\natexlab{}.
\newblock \showarticletitle{Intriguing properties of neural networks}.
\newblock \bibinfo{journal}{{\em arXiv preprint arXiv:1312.6199\/}}
  (\bibinfo{year}{2013}).
\newblock


\bibitem[\protect\citeauthoryear{Taigman, Yang, Ranzato, and Wolf}{Taigman
  et~al\mbox{.}}{2014}]%
        {taigman2014deepface}
\bibfield{author}{\bibinfo{person}{Yaniv Taigman}, \bibinfo{person}{Ming Yang},
  \bibinfo{person}{Marc'Aurelio Ranzato}, {and} \bibinfo{person}{Lior Wolf}.}
  \bibinfo{year}{2014}\natexlab{}.
\newblock \showarticletitle{Deepface: Closing the gap to human-level
  performance in face verification}. In \bibinfo{booktitle}{{\em Proceedings of
  the IEEE conference on computer vision and pattern recognition}}.
  \bibinfo{pages}{1701--1708}.
\newblock


\bibitem[\protect\citeauthoryear{{Tram{\`e}r}, {Kurakin}, {Papernot},
  {Goodfellow}, {Boneh}, and {McDaniel}}{{Tram{\`e}r} et~al\mbox{.}}{2018}]%
        {tram2018ensemble}
\bibfield{author}{\bibinfo{person}{F. {Tram{\`e}r}}, \bibinfo{person}{A.
  {Kurakin}}, \bibinfo{person}{N. {Papernot}}, \bibinfo{person}{I.
  {Goodfellow}}, \bibinfo{person}{D. {Boneh}}, {and} \bibinfo{person}{P.
  {McDaniel}}.} \bibinfo{year}{2018}\natexlab{}.
\newblock \showarticletitle{Ensemble Adversarial Training: Attacks and
  Defenses}.
\newblock \bibinfo{journal}{{\em 2018 ICLR\/}}  \bibinfo{volume}{arXiv preprint
  arXiv:1705.07204} (\bibinfo{year}{2018}).
\newblock


\bibitem[\protect\citeauthoryear{Wang and Banerjee}{Wang and Banerjee}{2014}]%
        {wang2014bregman}
\bibfield{author}{\bibinfo{person}{Huahua Wang} {and} \bibinfo{person}{Arindam
  Banerjee}.} \bibinfo{year}{2014}\natexlab{}.
\newblock \showarticletitle{Bregman Alternating Direction Method of
  Multipliers}.
\newblock In \bibinfo{booktitle}{{\em Advances in Neural Information Processing
  Systems 27}}, \bibfield{editor}{\bibinfo{person}{Z.~Ghahramani},
  \bibinfo{person}{M.~Welling}, \bibinfo{person}{C.~Cortes},
  \bibinfo{person}{N.~D. Lawrence}, {and} \bibinfo{person}{K.~Q. Weinberger}}
  (Eds.). \bibinfo{publisher}{Curran Associates, Inc.},
  \bibinfo{pages}{2816--2824}.
\newblock
\showURL{%
\url{http://papers.nips.cc/paper/5612-bregman-alternating-direction-method-of-multipliers.pdf}}


\bibitem[\protect\citeauthoryear{Xie, Wang, Zhang, Ren, and Yuille}{Xie
  et~al\mbox{.}}{2017}]%
        {xie2017mitigating}
\bibfield{author}{\bibinfo{person}{Cihang Xie}, \bibinfo{person}{Jianyu Wang},
  \bibinfo{person}{Zhishuai Zhang}, \bibinfo{person}{Zhou Ren}, {and}
  \bibinfo{person}{Alan Yuille}.} \bibinfo{year}{2017}\natexlab{}.
\newblock \showarticletitle{Mitigating adversarial effects through
  randomization}.
\newblock \bibinfo{journal}{{\em arXiv preprint arXiv:1711.01991\/}}
  (\bibinfo{year}{2017}).
\newblock


\end{thebibliography}
